%% file: main_arxiv.tex
\documentclass[12pt,reqno]{amsart}
\pdfoutput=1

\def\wt{\widetilde}
\def\wh{\widehat}
\def\<{\langle}
\def\>{\rangle}
\newcommand{\nats}{\mathbb N}

\newtheorem{inftheorem}{Informal Theorem}

\newcommand{\reals}{\mathbb R}
\newcommand{\calF}{\mathcal{F}}

\newcommand{\calL}{\mathcal{L}}

\newcommand{\calN}{\mathcal{N}}

\newcommand{\calQ}{\mathcal{Q}}

\newcommand{\calS}{\mathcal{S}}
\newcommand{\calT}{\mathcal{T}}
\newcommand{\calU}{\mathcal{U}}

\newcommand{\calW}{\mathcal{W}}
\newcommand{\calX}{\mathcal{X}}
\newcommand{\calY}{\mathcal{Y}}

\numberwithin{equation}{section}
\usepackage[tt=false]{libertine}
\usepackage{mathtools}

\usepackage{amssymb,mathrsfs}
\usepackage[varbb]{newpxmath}

\let\savedbigtimes\bigtimes
\let\bigtimes\relax
\usepackage{mathabx} 
\let\bigtimes\savedbigtimes

\usepackage[margin=1in]{geometry}
\usepackage{graphicx}
\usepackage{enumerate}
\usepackage{bbm}
\usepackage{hyperref,color}
\usepackage[capitalize,nameinlink]{cleveref}
\usepackage[dvipsnames]{xcolor}
\hypersetup{
	colorlinks=true,
	pdfpagemode=UseNone,
    citecolor=purple,
    linkcolor=purple,
    urlcolor=black,
	pdfstartview=FitW
}
\usepackage{appendix}
\usepackage{caption}
\RequirePackage{subfigure}

\crefname{appsec}{Appendix}{Appendices}
\usepackage{tikz}
\usepackage{bm}
\usepackage{mathtools}
\usepackage{color-edits}
\addauthor[Manolis]{mz}{Bittersweet}
\addauthor[Ilias]{iz}{green}
\addauthor[Alkis]{ak}{blue}

\newtheorem{theorem}{Theorem}[section]
\newtheorem{proposition}[theorem]{Proposition}
\newtheorem{lemma}[theorem]{Lemma}
\newtheorem{corollary}[theorem]{Corollary}

\theoremstyle{definition}
\newtheorem{definition}[theorem]{Definition}

\newtheorem*{assumption*}{Assumption}
\newtheorem{remark}[theorem]{Remark}

\crefname{lemma}{Lemma}{Lemmas}
\crefname{theorem}{Theorem}{Theorems}
\crefname{definition}{Definition}{Definitions}
\crefname{fact}{Fact}{Facts}
\crefname{claim}{Claim}{Claims}
\crefname{proposition}{Proposition}{Propositions}

\newcommand{\E}{\mathbb{E}}
\newcommand{\Var}{\mathrm{Var}}

\DeclareMathOperator*{\argmin}{arg\,min}
\newcommand{\norm}[1]{\left\lVert #1 \right\rVert}

\newcommand{\eps}{\varepsilon}
\renewcommand{\epsilon}{\varepsilon}
\def\l{\ell}

\newcommand{\beq}{\begin{equation}}
\newcommand{\eeq}{\end{equation}}

\definecolor{myC}{rgb}{0, 255, 255}
\definecolor{myY}{rgb}{204, 204, 0}
\definecolor{myM}{rgb}{255, 0, 255}
\definecolor{secinhead}{RGB}{249,196,95}
\definecolor{lgray}{gray}{0.8}

\usepackage{appendix}
\crefname{appsec}{Appendix}{Appendices}

\begin{document}

\title
[Transfer Learning Beyond Bounded Density Ratios]
{Transfer Learning Beyond Bounded Density Ratios}
\author
[A. Kalavasis, I. Zadik, M. Zampetakis]{Alkis Kalavasis$^\star$, Ilias Zadik$^{\circ,\star}$ and Manolis Zampetakis$^{\dagger,\star}$}
\thanks{\raggedright$^\circ$Department of Statistics and Data Science, Yale University;
$^\dagger$Department of Computer Science, Yale University.
$^\star$Institute of Foundations of Data Science, Yale University;\\
Email: \texttt{\{alvertos.kalavasis, ilias.zadik, manolis.zampetakis\}@yale.edu}}

\begin{abstract}
\small
  We study the fundamental problem of \textit{transfer learning} where a learning algorithm collects data from some source distribution $P$ but needs to perform well with respect to a different target distribution $Q$. A standard change of measure argument implies that transfer learning happens when the density ratio $dQ/dP$ is bounded. Yet, prior thought-provoking works by Kpotufe and Martinet (COLT, 2018) and Hanneke and Kpotufe (NeurIPS, 2019) demonstrate cases where the ratio $dQ/dP$ is unbounded, but transfer learning is possible.

  In this work, we focus on transfer learning over the class of \emph{low-degree polynomial estimators}. Our main result is a general transfer inequality over the domain $\mathbb{R}^n$, proving that non-trivial transfer learning for low-degree polynomials is \emph{possible under very mild assumptions}, going well beyond the classical assumption that $dQ/dP$ is bounded. For instance, it always applies if $Q$ is a log-concave measure and \emph{the inverse ratio} $dP/dQ$ is bounded. To demonstrate the applicability of our inequality, we obtain new results in the settings of: (1) the classical truncated regression setting, where $dQ/dP$ equals infinity, and (2) the more recent out-of-distribution generalization setting for in-context learning linear functions with transformers.

  We also provide a discrete analogue of our transfer inequality on the Boolean Hypercube $\{-1,1\}^n$, and study its connections with the recent problem of Generalization on the Unseen of Abbe, Bengio, Lotfi and Rizk (ICML, 2023). Our main conceptual contribution is that the maximum influence of the error of the estimator $\hat{f}-f^\star$ under $Q$, $\mathrm{I}_{\max}(\hat{f}-f^\star)$, acts as a sufficient condition for transferability; when $\mathrm{I}_{\max}(\hat{f}-f^\star)$ is appropriately bounded, transfer is possible over the Boolean domain. 

  The main techniques for proving our results are the Carbery-Wright anti-concentration inequality and the invariance principle of Mossel, O'Donnell and Oleszkiewicz (Annals of Mathematics, 2010).
\end{abstract}

\maketitle

\date{\today}

\newpage

\input{main_paper.tex}

\end{document}

%% file: main_paper.tex
%\tableofcontents

\section{Introduction}
Transfer learning is a pivotal concept in Statistics and Machine Learning that addresses the challenge of applying knowledge gained from one domain (the \emph{source}) to improve performance in a different, yet related domain (the \emph{target}). This approach is particularly relevant in scenarios where the distribution of data in the target domain, denoted as \(Q\), differs from that of the source domain, characterized by distribution \(P\). A rich body of research, including the works by \cite{hanneke2019value,hanneke2023limits, yang2013theory, redko2020survey,tripuraneni2020theory,ben2010theory,ben2006analysis,pathak2022new}, and others, has extensively explored transfer learning from various angles, underscoring its theoretical and practical significance. This concept, also known as Out-Of-Distribution (OOD) Generalization or Domain Adaptation, is crucial for understanding how algorithms can be trained on data from the source domain and then adeptly applied to the target domain. The quest to pinpoint the exact nature of distributional similarity that facilitates successful transfer remains a central, yet not well-understood, inquiry in this domain.

The concept of transfer learning can be encapsulated within the framework of a standard regression problem. Here, our goal is to identify a \emph{regressor} \(f\) that aims to minimize the mean squared error \(\E_{x \sim P}[(f(x) - f^\star(x))^2]\), using samples in the form \((x, y)\), where \(x\) is drawn from distribution \(P\) and $\E[y | x] = f^{\star}(x)$. The twist in transfer learning emerges when the objective shifts towards minimizing the expected error \(\E_{x \sim Q}[(f(x) - f^\star(x))^2]\), despite the fact that our available samples are drawn from \(P\). This raises the pivotal question: is it possible to establish a relationship between the expected errors \(\E_{x \sim P}[(f(x) - f^\star(x))^2]\) and \(\E_{x \sim Q}[(f(x) - f^\star(x))^2]\)? In essence, this inquiry delves into the capability of the regressor \(f\) to \emph{extrapolate} effectively to samples that are out-of-distribution, originating from the target distribution \(Q\), based on its training with the source distribution \(P\).\footnote{Throughout the paper, when talking about the Euclidean domain, we consider a probability space $(\Omega, \calF, (P,Q))$ where $\Omega = \reals^n$ and for clarity we will assume that both distributions $P$ and $Q$ are \emph{continuous} distributions.}
A natural approach to compare the mean squared error with respect to $P$ with the mean squared error with respect to $Q$ is to apply a change of measure:
\begin{align}
\E_{x \sim Q} [(f(x) - f^\star(x))^2]
& =
\E_{x \sim P} \left[ \frac{Q(x)}{P(x)} (f(x) - f^\star(x))^2 \right] \nonumber \\
& \leq 
\norm{\frac{dQ}{dP}}_{\infty} \E_{x \sim P}[(f(x) - f^\star(x))^2]\,,\label{eq:naive}
\end{align}
where $\| dQ / dP\|_\infty = \sup_{x \in \mathrm{supp}(Q)} Q(x)/P(x)$ 
\footnote{We note that, instead of the ratio $\|dQ/dP\|_\infty,$ applications of Hölder's inequality
can yield different notions of divergence $\|dQ/dP\|_r \triangleq D_r(Q \parallel P) = (\E_{x \sim P}[(Q(x)/P(x))^r])^{1/r}$ for $r \geq 1$ \cite{renyi1961measures} in \eqref{eq:naive}.}.
Inequality \eqref{eq:naive}, which is meaningful only when $Q$ is absolutely continuous with respect to $P$, comes with a simple intuition: we can think of the target distribution $Q$ as the ``nature" and the source distribution $P$ as the ``model" aiming to approximate $Q$. {If the (asymmetric) divergence of $Q$ from $P$ (e.g., the $L_\infty$-norm of the density ratio $dQ/dP$) is bounded, then, for any event $A$, if $P(A)=o(1)$, then it also holds $Q(A)=o(1)$. Based on that, it is natural to expect that replacing the ``nature'' $Q$ with the ``model'' $P$ is not very harmful; for instance, if  $f \approx f^*$ with probability $1-o(1)$ under $P$, then $f \approx f^*$ with probability $1-o(1)$ under $Q.$ } This simple and intuitive idea has inspired a line of work that explores transfer learning assuming an upper bound on some norm of the likelihood ratio $dQ / dP$ \cite{sugiyama2012density,sugiyama2012machine,sugiyama2007covariate,kpotufe2017lipschitz,que2013inverse,kpotufe2018marginal,cortes2010learning,ma2023optimally,pathak2022new}.

This naturally leads us to question the validity of the assumption that the ratio $\|dQ / dP \|_{\infty}$ remains within a bounded range. Unfortunately, this ratio can potentially be unbounded in many relevant scenarios such as: (i) truncation and censoring, which are fundamental challenges in Statistics, with historical roots tracing back to the works of Galton \cite{galton1898examination} and Pearson \cite{Pearson1902}; (ii) the innovative Generalization On The Unseen framework proposed by \cite{abbe2023generalization}; (iii) straightforward distribution shifts like translation, exemplified by the transition from a uniform distribution \(P = \mathcal{U}([0,1])\) to \(Q = \mathcal{U}([1,2])\); and (iv) matrix completion tasks involving data that is missing not at random, along with more complex combinatorial distribution shifts as discussed by \cite{simchowitz2023tackling}. Additionally, even in instances where the ratio is bounded, it can still attain exceedingly high values. For example, the ratio of two isotropic Gaussian distributions with different means, \(\mu_1\) and \(\mu_2\), escalates \emph{exponentially} with the squared distance \(\|\mu_1 - \mu_2\|_2^2\) making inequality \eqref{eq:naive} very weak even in this simple setting.
Due to the many instances where the ratio $dQ/dP$ is unbounded, a pertinent question arises:
\begin{center}
    \emph{Is transfer learning possible when $\|dQ / dP \|_r \to \infty$?
    } 
\end{center}

This question is further inspired by insightful findings from Kpotufe and Martinet \cite{kpotufe2018marginal}, as well as Hanneke and Kpotufe \cite{hanneke2019value}, within the framework of nonparametric classification. These studies highlight scenarios in which the ratio $\|dQ / dP\|_r$ is unbounded, yet still transfer successfully happens. Our research is motivated by these observations and seeks to delve into broader circumstances where, despite the large magnitude of this ratio, transfer learning is feasible and effective.

\subsection{Summary of Contributions}
In this paper, we consider the expressive and well-studied class of \emph{low-degree polynomials}. Our main results are general transfer inequalities, for functions in this class, under mild assumptions on the distribution pair $P$ and $Q$, going well beyond the boundedness of the ratio $dQ/dP$. In fact, we can almost always bound the error with respect to $Q$ as a function of the error with respect to $P$.\\

\noindent\textbf{Transferability of Low-degree Polynomials in the Euclidean Domain.}
When the domain of the covariates $x$ is the $n$-dimensional Euclidean space, we prove the following theorem.
\begin{inftheorem}[See \Cref{theorem:main}] \label{inftheorem:main}
  For any pair of degree-$d$ polynomials $f, f^{\star} : \reals^n \to \reals$, any log-concave distribution $Q$ and any {continuous} distribution $P$, it holds that
\begin{equation}
\E_Q [(f(x) - f^{\star}(x))^2] \leq C_d \cdot \norm{ \frac{dP}{dQ} }_\infty^{2 d} \cdot \E_P [(f(x) - f^{\star}(x))^2]\,,
\label{eq:main1}
\end{equation} 
where $C_d>0$ is a constant depending only on the degree $d$. Furthermore, even when $Q$ is not log-concave, it holds that
\begin{equation}
\E_Q [(f(x) - f^{\star}(x))^2] \leq C_d \cdot \inf_{\mu \in \calL} \norm{ \frac{dQ}{d\mu} }_\infty \norm{ \frac{dP}{d\mu} }_\infty^{2 d} \cdot \E_P [(f(x) - f^{\star}(x))^2]\,,
\label{eq:main1:2}
\end{equation} 
where the infimum is over the set $\calL$ of log-concave distributions over $\reals^n$ and $P$, $Q$ are arbitrary\footnote{Observe that it is necessary for the inequality not being void to take the infimum over measures $\mu \in \calL$ whose support contains $\mathrm{supp(P) \cup \mathrm{supp}(Q)}$. We refer to Section \ref{sec:mainRes} for more details.}.
\end{inftheorem}
\noindent Let us give some explanation of the results of the Theorem above.
\begin{description}
  \item[Inequality \eqref{eq:main1} (The importance of $dP/dQ$)] If we compare \eqref{eq:main1} with \eqref{eq:naive} we observe the surprising fact that in \eqref{eq:main1} the \textbf{``inverse'' ratio} $$dP/dQ$$ arises. {Comparing the boundedness assumption on $dP/dQ$ with the naive assumption that $dQ/dP$ is bounded, our inequality proves something perhaps surprising. Even in cases where (i) the support of the source $P$ does not contain that of $Q$, or, more generally, (ii) $P$ assigns much smaller probability mass to some event compared to $Q$ (so that $dQ/dP$ becomes extremely large), as long as $dP/dQ$ is bounded,
  the error with respect to $P$ still transfers to an error with respect to $Q$ for low-degree polynomials.
  }
  
  \item[Inequality \eqref{eq:main1:2}] Combining \eqref{eq:main1} with the naive bound of \eqref{eq:naive}, we get inequality \eqref{eq:main1:2}. This inequality allows us to show transferability even between distributions $P$ and $Q$ that have \textbf{completely disjoint supports}!  
  Furthermore, the term $$ \inf_{\mu \in \calL} \norm{ \frac{dQ}{d\mu} }_\infty \norm{ \frac{dP}{d\mu} }_\infty^{2 d} $$ defines a novel quantitative notion of ``transfer difference'' between $P$ and $Q$ that allows for transferability from $P$ to $Q$ {for low-degree polynomials}.
\end{description}

Informal Theorem \ref{inftheorem:main} is a general transferability result that can be applied in various settings. To illustrate this we show the following applications of this Theorem.
\begin{enumerate}
    \item \textbf{Truncated Regression.} We show that our inequality can be applied to regression problems from truncated samples, a fundamental problem in statistical literature; see as an example the recent work of \cite{daskalakis2019computationally} and references therein. Our result in \Cref{cor:truncated} shows that, in such settings, the MSE in the truncated domain upper bounds the MSE in the whole population in very broad regression settings.
    \item \textbf{Distribution Shift of Linear Attention.} In-context learning linear functions with linear attention was theoretically explored in \cite{zhang2023trained} under Gaussianity assumptions during the training of the model. In \Cref{sec:app2-main}, we extend their statistical results {on distribution shifts for in-context learning linear functions with linear attention} in a much more general setting under mild assumptions.
    \item \textbf{Translations of Log-concave Densities.} In \Cref{sec:examples}, we show that using our inequality we can bound the transfer error when the source distribution is log-concave and the target distribution is a translation of the source. This illustrates that our result can handle simple transformations of the data. Revisiting the example of two isotropic Gaussian distributions with different means, \(\mu_1\) and \(\mu_2\), (discussed in the Introduction above) the transfer coefficient that is implied by Informal Theorem \ref{inftheorem:main} is of order $\norm{\mu_1 - \mu_2}_2^d$, as opposed to  $\exp\left(\norm{\mu_1 - \mu_2}^{\Theta(1)}_2\right)$ which is what we get by applying the naive bound \eqref{eq:naive}.
\end{enumerate}

\noindent\textbf{Transferability of Low-degree Polynomials in the Boolean Domain.}
Next we focus on discrete domains and in particular on the Boolean domain $\{-1, 1\}^n$. In this setting, we need some additional assumptions than just that $f$, $f^{\star}$ being polynomials. The main reason is that any function over $\{-1, 1\}^n$ can be expressed as a polynomial in the Fourier(-Walsh) basis and transferability can be easily shown not to hold for all functions.
{
Given our inequality in the continuous domain, one may wonder if transfer holds as long as we focus on low-degree Boolean functions.
However, it is not difficult to see that the low-degree assumption is not anymore sufficient for transfer learning in discrete domains. For instance, let us consider the (shifted) dictator (in direction $i=1$) function $f(x) = x_1 + 1$, which is a low-degree non-negative polynomial, $Q$ be the uniform distribution over $\{-1,1\}^n$ and $P$ be the uniform distribution on the subset of the Boolean hypercube with $\{x_1 = -1\}.$ It is evident that $\E_P f = 0,$ while $\E_Q f = 1$,
%\textcolor{red}{it holds that$(x_1-1)^2=0$ on $\{x_1=1\}$}it is easy to design a source $P$ and a target $Q$ with $\|d P/ dQ\|_\infty < \infty$ so that $\E_P[x_i] = 0$ but $\E_Q[x_i] > 0$, 
which means that transfer from samples of $P$ to samples of $Q$ does not hold, i.e., no inequality of the form $\E_Q f^2=O(\E_P f^2)$ is possible.}

{
A reader familiar with Boolean Fourier analysis may realize that a potential underlying issue with the dictator function is its \emph{high influence}\footnote{Recall that for a function $f: \{-1,1\}^n \to \{-1,1\}$ the influence $\mathrm{Inf}_i(f)$ in direction $i$ is the probability that when we flip the value of the $i$-th coordinate,
the value of $f$ is flipped as well. For instance, for $f(x_1,...,x_n) = x_i$, we have that $\mathrm{Inf}_i(f) = 1$ and $\mathrm{Inf}_{j}(f) = 0, j \neq i$.
This notion of sensitivity naturally extends to functions $f : \{-1,1\}^n \to \reals$, as we will see in \Cref{sec:boolean}.} (in direction $i = 1)$; it is in fact easy to construct more counterexamples with other Boolean functions of high influence in some direction.
This perhaps suggests that the maximum influence of low-degree functions over all directions $i \in [n]$ could be an important parameter for general transfer.} Focusing on low-degree Boolean functions, we actually show that having a high influence in some direction is the \emph{only bottleneck} for transferability in the Boolean domain.

\begin{inftheorem}[See \Cref{thm:transfer}] \label{inftheorem:main:2}
Let $Q$ be the uniform distribution over $\{-1,1\}^n$ and $P$ be the uniform distribution conditioned only on a subset $S \subseteq \{-1,1\}^n$. For any degree-$d$ polynomials $f, f^{\star} : \{-1,1\}^n \to \reals$ that satisfy $\mathrm{Inf}_i(f - f^{\star}) = Q(S)^{O(d)}$ for all $i \in [n]$, it holds that
\begin{equation}
    \E_Q \left[(f(x) - f^\star(x))^2 \right]
    \leq 
    C_d \cdot 
    Q(S)^{-2d} 
    \cdot 
    \E_P \left[(f(x) - f^\star(x))^2 \right]\,,
\end{equation}
where $C_d$ is a constant depending only on the degree $d$. 
\end{inftheorem}

One application of our transferability Theorem in the Boolean domain is the Generalization On the Unseen setting of \cite{abbe2023generalization}. 
In \Cref{section:boolean-discussion}, we revisit the problem of \cite{abbe2023generalization}
of training an (overparameterized) diagonal linear network $f_{\mathrm{NN}}$ (cf. \Cref{def:dll}) in order to learn a simple linear function $f^\star : \{-1,1\}^n \to \reals$ via gradient methods. The catch is that the observed features $x$ are generated by $P$ which is the uniform distribution conditional on the set $S = \{x \in \{-1,1\}^n | x_k = 1\}$ for some $k \in [n]$.
We refer to $S$ as the \emph{seen} part of the Boolean hypercube.
The goal is to understand how this trained model \emph{generalizes on the unseen}, i.e., how small is the risk of $f_{\mathrm{NN}} - f^\star$ under the uniform distribution over the whole Boolean domain. 
This is an instantiation of our transfer learning framework.
A relatively simple corollary of Informal Theorem \ref{inftheorem:main:2} and techniques borrowed from \cite{abbe2023generalization}, described below in \Cref{thm:DLN-transfer}, is that diagonal linear networks  trained with gradient flow exhibit the following interesting behavior: there exists a critical time $t^\star=t^{\star}_n=\Omega(\log n)$ for which for any $0 < t < t^\star$, the maximum influence of $f_{\mathrm{NN}} - f^\star$ is sufficiently small, our transfer inequality is applicable and therefore \emph{ generalization on the unseen} takes place. However, for $t > t^\star$, the maximum influence becomes larger than the critical threshold of Informal Theorem \ref{inftheorem:main:2} and our transfer inequality is not applicable. The fact that our inequality is not applicable for ``infinite'' $t$ is in agreement with \cite[Theorem 3.11]{abbe2023generalization} where it is proven that gradient flow with infinite time has non-perfect generalization on the unseen.
For further details, see \Cref{section:boolean-discussion}.\\
 
\noindent\textbf{Techniques.} Our main tool in the proof of Informal Theorem \ref{inftheorem:main} is the use of \emph{anti-concentration} properties of polynomials under log-concave densities. To the best of our knowledge, this is the first time that tools such as the Carbery-Wright inequality \cite{carbery2001distributional} are used in the general abstract context of transfer learning, with our inspiration being the special case of truncated statistics \cite{cohen1991truncated,daskalakis2018efficient}.
For the proof of Informal Theorem \ref{inftheorem:main:2}, we rely on the invariance principle of \cite{mossel2005noise} for product measures combined with the Carbery-Wright inequality. We refer to \Cref{sec:boolean} for more technical details.
\subsection{Empirical Motivation: Polynomial Regression vs. Deep ReLU Networks}
Before the formal statements of our results, we present a simple motivating example, depicted in \Cref{fig:intro}, which illustrates our theoretical findings on the transferability properties of polynomials and compares them with the empirical transferability properties of neural networks. To test the robustness of our results in this example we consider the case where the dimension $n = 2$ and $$f^\star(x, y) = \sin(2\pi x) \cdot \sin(2\pi y),$$ which is not a low-degree polynomial (but sufficiently smooth). We also pick the source distribution $P$ to be the uniform distribution over $[0,1] \times [-1,1]$.

Given i.i.d. samples from $P$ labeled using $f^\star$, we train two models: \begin{itemize}
    \item[(a)] $f_1$ being a low-degree polynomial and,

    \item[(b)] $f_2$ being a deep neural network with ReLU activation functions,
\end{itemize}  with the common objective of minimizing the mean squared error over the source $P$ $\E_P[(f_i - f)^2], i=1,2$.

We plot the resulting functions $f_1$ and $f_2$ over the square $[-5, 5]^2$ to test how well $f_1$ and $f_2$ are approximating $f^{\star}$ outside of $[0,1] \times [-1,1]$, which helps us understand the transferability of $f_1$ and $f_2$ for a distribution $Q$ that is supported also \emph{outside of} $[0,1] \times [-1,1]$. As a note, clearly, in this experimental setup, the ratio $\|dQ/dP\|_\infty = \infty$ since $\mathrm{supp}(P) \subsetneq \mathrm{supp}(Q)$ and therefore no transfer claims can be based on the naive arguments using this ratio.

\begin{figure}[ht!]
    \centering
    \subfigure[$f^{\star}$]{\includegraphics[scale=0.23]{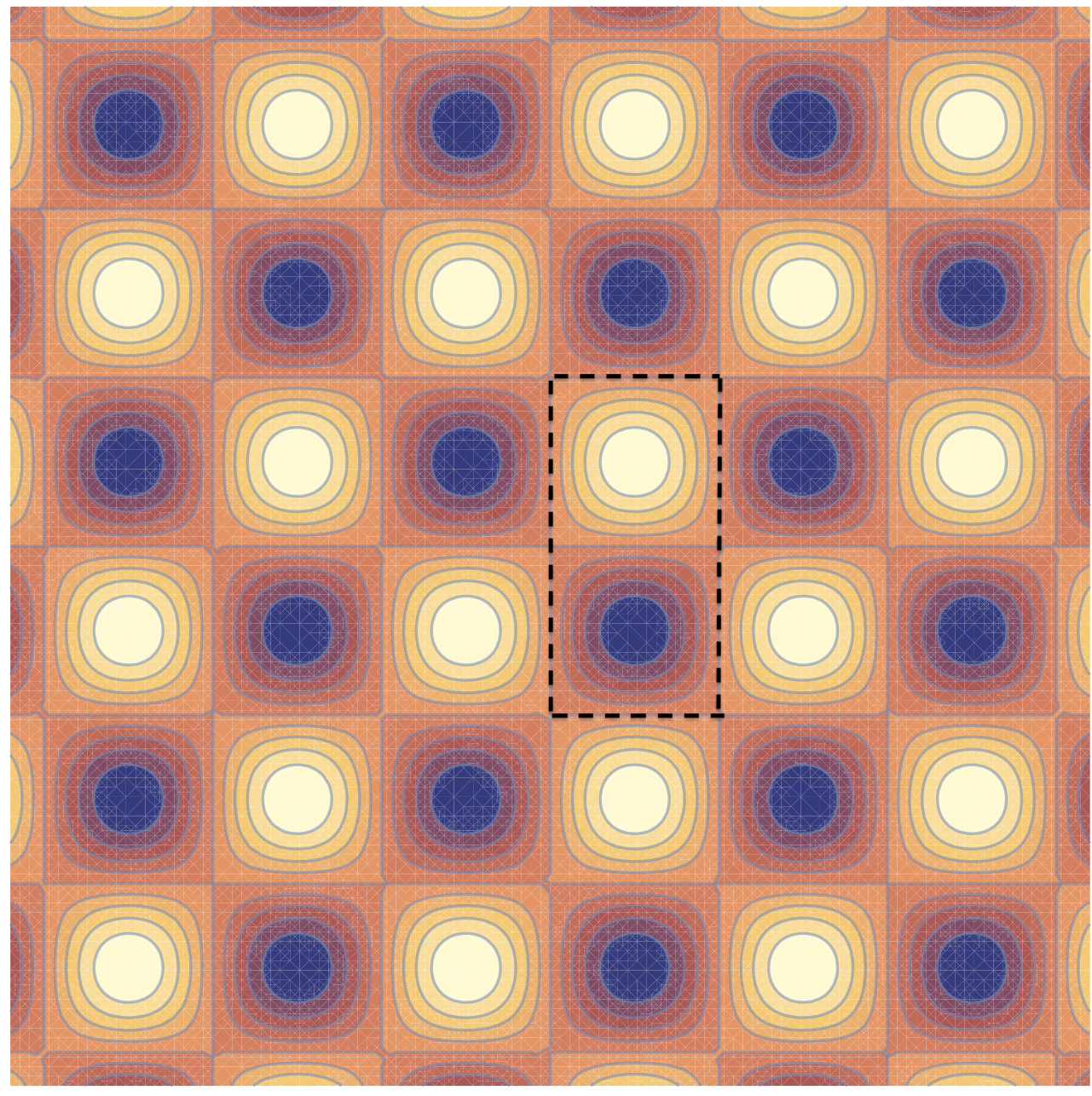}}\label{fig:intro:1}
~
\subfigure[$f_1$]{
\includegraphics[scale=0.23]{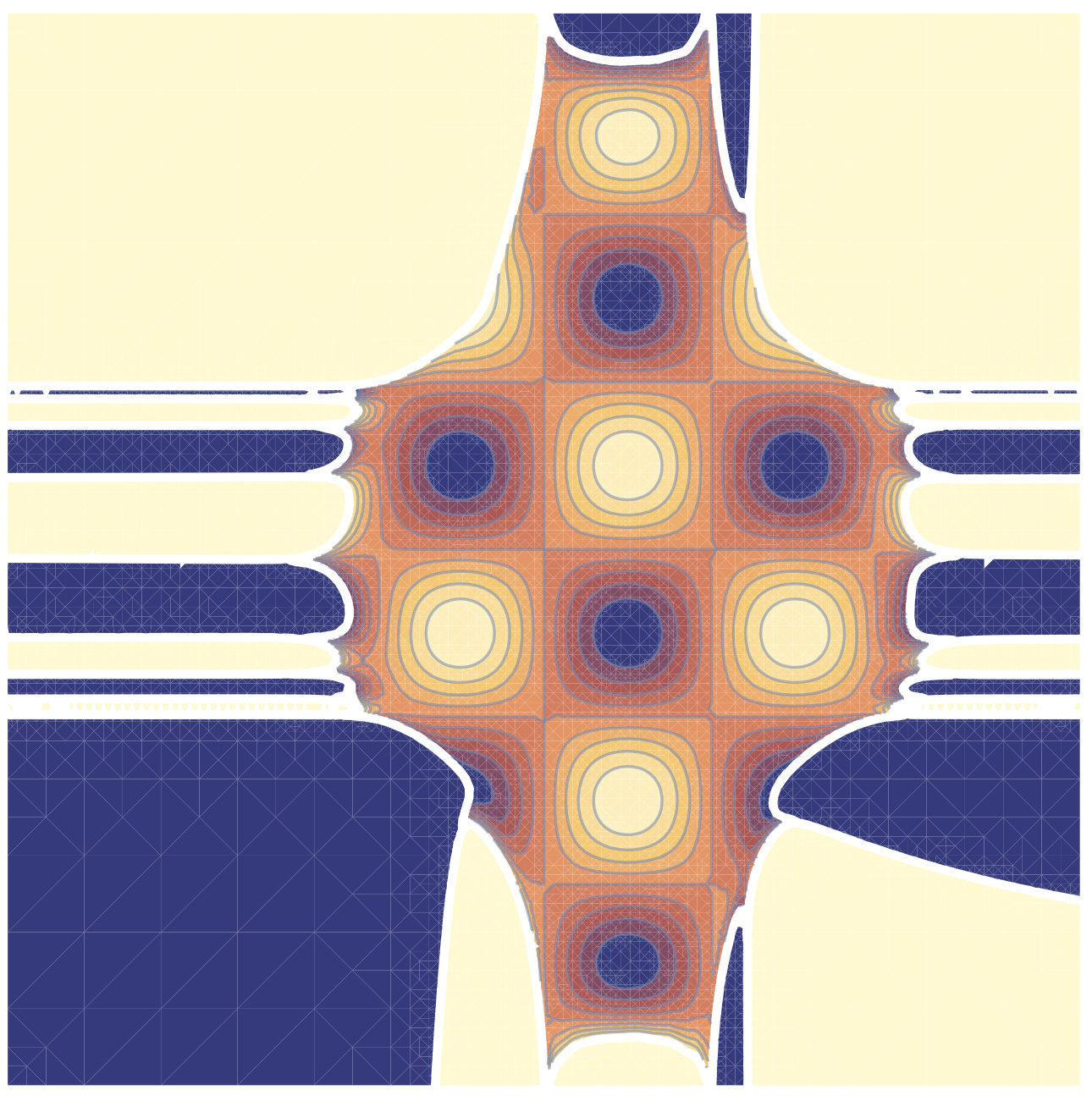}} \label{fig:intro:2}
~
\subfigure[$f_2$]{
\includegraphics[scale=0.23]{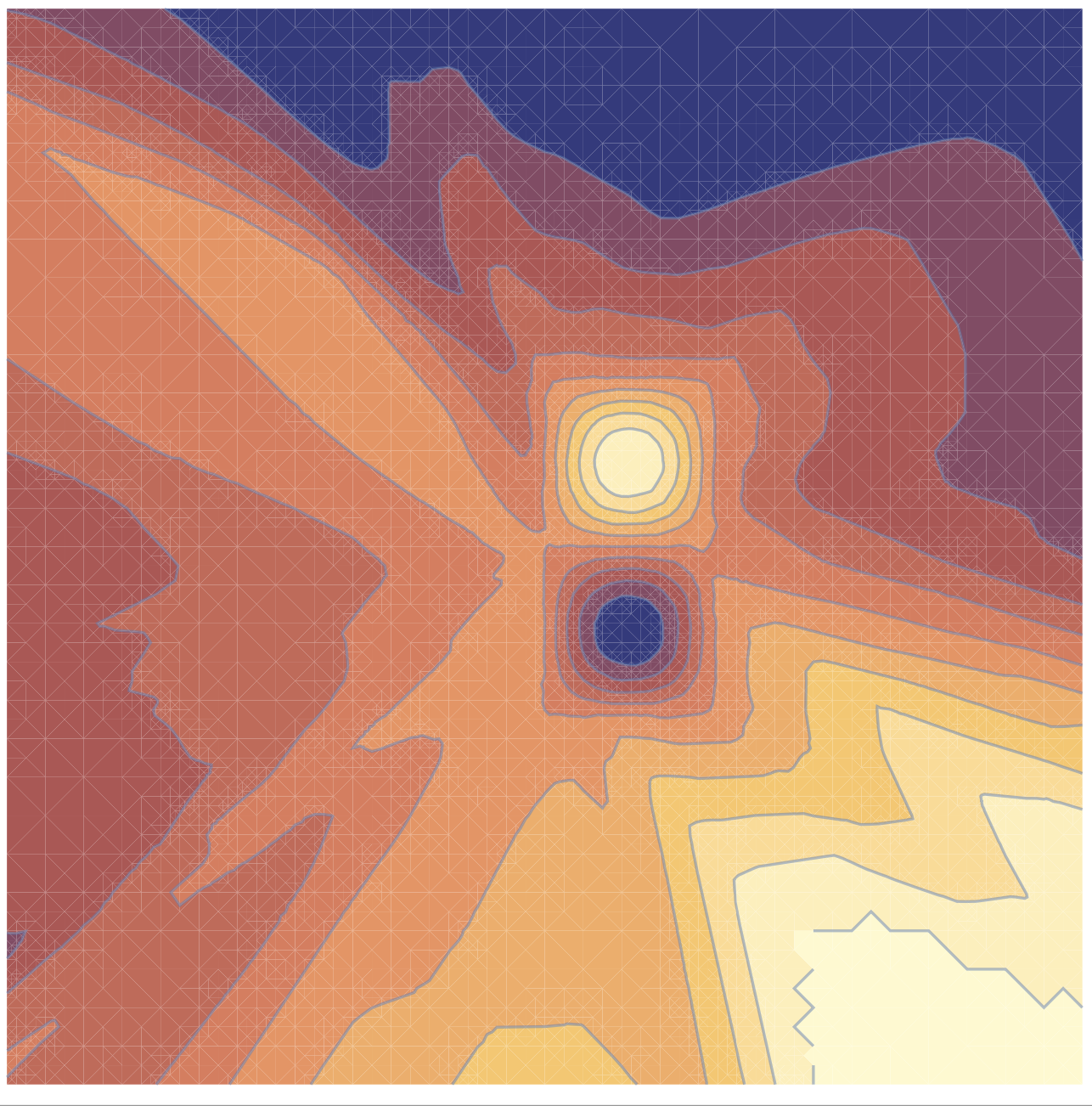}} \label{fig:intro:3}
\caption{Transferability of polynomial and neural network estimators with $P = \calU([0,1]\times[-1,1])$ (inside the dotted rectangle in (a)). (a) contour plot of $f^\star(x,y) = \sin(2\pi x)\sin(2\pi y)$, (b) contour plot of a degree-20 polynomial regressor $f_1$, (c) contour plot of a 6-layer size-110 ReLU network $f_2$ (see also \Cref{sec:exp}).}
    \label{fig:intro}
\end{figure}
As one can directly see from the figures, polynomial regression surprisingly achieves a \textbf{highly non-trivial generalization} on a significant part of the \textbf{unseen part} of the domain, while the ReLU network completely fails!

This is in agreement with our Informal Theorem \ref{inftheorem:main}, and it illustrates that a theorem similar to Informal Theorem \ref{inftheorem:main} cannot hold in general for neural networks. We refer to \Cref{sec:exp} for more details.

\subsection{Related Work}

Transfer learning works have mostly focused on classification and regression tasks 
\cite{mansour2009domain,ben2010theory,mansour2012multiple,david2010impossibility,yang2013theory,yang2018bounds,hanneke2019value,hanneke2023limits,kpotufe2018marginal,mousavi2020minimax}. We note that the recent work of \cite{hanneke2019value}
introduced the notion of
\emph{transfer exponents}  which, while being an interesting assumption on transferability, does not indicate when transfer is actually possible (which is the focus of our inequalities).
For further references,  
see the survey papers \cite{weiss2016survey,zhuang2020comprehensive,redko2020survey}.
Moreover, transfer learning has been a topic of interest for the statistics community \cite{cai2021transfer,reeve2021adaptive,li2022transfer}. Finally, recent work by \cite{klivans2023testable} studies transfer learning through the lens of the testable learning framework, while \cite{kalai2021efficient} studies the same problem through the lens of PQ learning \cite{goldwasser2020beyond}. We continue with some other related topics.\\ 

\noindent\textbf{Generalization on the Unseen (GOTU).}
GOTU \cite{abbe2022learning,abbe2023generalization}
is a special case of the OOD setting. In GOTU, the algorithm has access to one part of the distribution domain and completely omits the complement. The question is how a learner trained in this (seen) part can extrapolate/generalize on the unseen. \cite{abbe2023generalization} study the overparameterized setting and want to understand the implicit bias of SGD on learning a linear function by observing data only a subset of the hypercube. Among all 0 training loss functions, they show that the random feature model when trained with GD has bias towards the Min-Degree Interpolator, i.e., the function $f$ with 0 training loss and whose Fourier spectrum is concentrated in very low coefficients. Similar results are shown for diagonal linear NNs, a setting that we also study in \Cref{section:boolean-discussion}.\\

\noindent\textbf{Truncated Statistics.}
Truncated statistics \cite{galton1898examination,cohen1991truncated} is an area that is closely related to the focus of this work. In this setting, there is a truncation set $S$, which is a potentially unknown subset of the domain and some target underlying distribution (e.g., Gaussian).
The learner, whose goal is to estimate the true distribution,  observes a sample from the true model only if this falls in $S$ and observes no samples outside of $S$. For more details, we refer to \Cref{sec:app1}.
Designing efficient algorithms for truncated and censored problems has recently gained attention from the computational learning theory community \cite{daskalakis2018efficient,daskalakis2019computationally,daskalakis2020truncated,fotakis2020efficient,fotakis2021efficient,daskalakis2021efficient,fotakis2022perfect,de2023testing,nagarajan2020analysis,galanis2023learning,de2024detecting,daskalakis2021statistical}, with the Carbery-Wright inequality being a major tool for (most of) these results. \\

\noindent\textbf{Extrapolation in Neural Networks.}
Deep neural networks behave erratically when
presented with out-of-distribution (OOD) inputs \cite{guo2017calibration,jain2023data, shah2020pitfalls,kang2023deep,xu2020neural}. Transfer learning in general does not occur in NN training and hence our general inequalities are not applicable in this setting. The main intuition is due to the existence of implicit bias during the training of NNs via (S)GD methods (e.g., \cite{lyu2019gradient,ji2019implicit,chizat2020implicit}). We refer to the recent work by \cite{kang2023deep} for further justification on the inability of NNs to extrapolate. 
% The reasoning goes as follows. First, based on \cite{ji2020directional}, the resulting (sufficiently deep and wide) neural network will have low rank matrices $W_i$ in each layer.
% \cite{kang2023deep} observes that, for any layer $i$, weight matrices $W_i$ and network representations $\phi_i(x)$
% associated with training inputs $x \sim P_{train}$ often occupy low-dimensional subspaces with high overlap, i.e., $W_i$ has low rank (due to the implicit bias analysis) and $\| W_i \phi_i(x) \|_2 \approx $ large. However,
% when the network encounters OOD inputs $y \sim P_{test}$, we observe that their associated representations tend to have
% less overlap with the weight matrices compared to those from the training distribution, (particularly in
% the later layers), i.e., $\|W_i \phi_i(y)\|_2 \approx$ small. As a result, OOD representations tend to diminish in magnitude as they pass through
% the layers of the network, causing the network’s output to be primarily influenced by the accumulation of model constants (bias). Hence, as the data becomes more and more OOD, the model predictions on OOD inputs worsen. 
\\

\noindent\textbf{Anti-Concentration and Invariance Principle.}
Our main tool is the well-known Carbery-Wright inequality \cite{carbery2001distributional}. Similar inequalities are known in the continuous case (see e.g., \cite{kane2012structure,glazer2022anti}). 
There is a lot of work on designing similar anti-concentration inequalities in discrete domains \cite{littlewood1939number,erdos1945lemma,nguyen2013small,razborov2013real,meka2015anti,fox2021combinatorial}. 
The invariance principle \cite{mossel2005noise} is a key tool for anti-concentration in discrete domains. We shortly mention that tools such as anti-concentration and CLT-type structural results have been extensively applied in learning theory and analysis of Boolean functions (see \cite{harsha2009bounding,diakonikolas2010bounding,daskalakis2018efficient,daskalakis2015sparse,kalavasis2022learning,diakonikolas2013improved,klivans2009learning} and the references therein). 

\subsection{Overview}

We next present the formal results of our paper. In \Cref{sec:mainRes}, we present the necessary definitions that lead to the formal version of Informal Theorem \ref{inftheorem:main}. 
Then, in \Cref{sec:mainRes2}, we present the formal statement of Informal Theorem \ref{inftheorem:main:2}. Moreover, in \Cref{sec:app1}, we present the application of Informal Theorem \ref{inftheorem:main} in truncated regression that we highlighted above and, in \Cref{sec:app2-main}, we discuss the application of Informal Theorem \ref{inftheorem:main} on distribution shifts for in-context learning linear function with linear attention.
The rest of the paper is organized as follows: in \Cref{section:boolean-discussion},
 we discuss the application of Informal Theorem \ref{inftheorem:main:2} in the Generalization on the Unseen (GOTU) problem and, in \Cref{sec:exp}, we discuss some empirical observations. 
 In \Cref{app:proof-euc} we present the proofs for our results in the Euclidean domain and in \Cref{app:proof2} the associated proofs for the Boolean domain.
 In
 \Cref{sec:proofs}, we provide all the proofs omitted from our applications in the Euclidean and Boolean domains. \Cref{sec:examples} contains examples evaluating the transfer coefficient $\|dP/d\mu\|_\infty \cdot \|d Q/d \mu\|_\infty^d$ for natural families of distributions $(P,Q).$

\section{Transfer Learning for Low-Degree Polynomials in the Euclidean Domain} \label{sec:mainRes}

\subsection{Notation \& Definitions} We start with some definitions that we need in order to state our main Theorem in this section. 

\begin{definition}
\label{def:div}
Let $P,Q$ be two probability distributions over $\reals^n$. Let $P(x)$ (resp. $Q(x)$) be the probability density function of $P$ (resp. $Q)$ evaluated at $x \in \reals^n$. We define
\[
\left \|\frac{dP}{dQ} \right\|_{\infty}
\triangleq
\sup_{x \in \reals^n} \frac{P(x)}{Q(x)} \in [1,\infty]\,,
\]
and, for real $\alpha \geq 1$,
\[
D_\alpha(P \parallel Q) \triangleq
\left( \E_{x \sim Q}\left[\left(\frac{P(x)}{Q(x)}\right)^\alpha \right] \right)^{\frac{1}{\alpha}}\,.
\]
\end{definition}

\begin{definition}
A distribution is called log-concave if its density is of the form $\propto \exp(-f)$, where $f$ is convex. Let $\calL$ be the set of all log-concave distributions over $\reals^n$.
\end{definition}

\noindent

\subsection{Main Inequality} We are now ready to state our main result for the Euclidean domain.

\begin{theorem}
[Transferability of Polynomials]
\label{theorem:main}
Let $\calL$ be the space of log-concave probability distributions over $\reals^n$.
Consider two probability distributions $P,Q$ over $\reals^n$ and let $f : \reals^n \to \reals$ be a degree-$d$ polynomial. 
There exists an absolute constant $C$ such that for $\alpha,\beta \in [1,\infty]$ with $1/\alpha + 1/\beta = 1$, it holds that
\begin{equation}
\label{ineq:main}
\E_{Q} |f| \leq 
(Cd)^d 2^{d \beta}
\cdot 
\inf_{\mu \in \calL} 
D_{\alpha}(Q \parallel \mu)
D_\alpha(P \parallel \mu)^{\beta d} 
\cdot
\left(\E_P |f|^{\beta}
\right)^{1/\beta}\,.
\end{equation}
In particular, if $\alpha = \infty$, we get that
\begin{equation}
\E_Q |f|
\leq 
(2 C d)^d
\cdot 
\inf_{\mu \in \calL}
\left \| \frac{d Q}{d \mu} \right\|_\infty
\left \| \frac{d P}{d \mu} \right\|_\infty^{d}
\cdot 
\E_P |f|\,.
\end{equation}
\end{theorem}

If $Q$ is log-concave, then the above simplifies to the following.
\begin{corollary}
\label{cor:inequality}
Consider two probability distributions $P,Q$ over $\reals^n$ where $Q$ is log-concave. Assume that $\mathrm{supp}(P) \subseteq \mathrm{supp}(Q)$. Let $f : \reals^n \to \reals$ be a degree-$d$ polynomial. Then, for any $\alpha, \beta \in [1,\infty]$ with $1/\alpha + 1/\beta = 1$, it holds that
\begin{equation}
\E_Q |f|
\leq 
(C d)^d 2^{ d \beta}
\cdot 
D_\alpha(P \parallel Q)^{\beta d}
\cdot \left(\E_P |f|^\beta \right)^{1/\beta}\,.
\end{equation}
In particular, if $\alpha = \infty$, we get that
\begin{equation}
\E_Q |f|
\leq 
(2 C d)^d
\cdot 
\left \| \frac{d P}{d Q} \right\|_\infty^{d}
\cdot 
\E_P |f|\,.
\end{equation}

\end{corollary}

Our \Cref{theorem:main}, and in particular \Cref{cor:inequality}, perhaps unintuitively, propose that considering the inverse ratio $\| dP/ dQ \|_\infty$ allows for transfering when working with \emph{low-degree polynomials}. For the proof of \Cref{theorem:main}, we refer to \Cref{section:proof of main} and for \Cref{cor:inequality} to \Cref{app:proof1}.

{\Cref{theorem:main} allows us to transfer even in situations where both ratios are infinity. 
Consider, for instance, the simple example of distribution shift from source $P = \mathcal{U}([0,1])$ to target $Q = \calU([2,3])$. Our \Cref{theorem:main} proposes that we can use an intermediate log-concave measure $\mu$ that ``covers" both $P$ and $Q$ (e.g., $\mu = \calU([0,3])$ 
and obtain transfer results at a cost that scales with $\|d Q/ d \mu\|_\infty \cdot \| d P / d \mu \|_\infty^d$, where $d$ is the degree of the polynomial. This, perhaps counter-intuitive, phenomenon follows from the property of polynomials $f$ in the Euclidean space, that the probability density of the random variable $f(x)$ does not contain point masses when $x$ is drawn from a log-concave distribution. 
}

In \Cref{sec:examples}, we compute such transfer ratios for various $(P,Q)$ pairs.
An interesting example is
the setting where both $P$ and $Q$ are standard Gaussians with different means in $\reals^n$ (translation). One can show that the ratio $\|d Q/d P\|_\infty$ scales exponentially with the distance of the two means but, in the contrary, for  $\|dQ/d\mu\|_\infty \|dP /d\mu\|_\infty$ there exists a log-concave distribution $\mu$ so that the product scales polynomially with the means' distance.

\section{Transfer Learning for Low-Degree Polynomials in the Boolean Domain} \label{sec:mainRes2}
\label{sec:boolean}
In this section, we will provide a similar transfer 
inequality for discrete domains.
For simplicity, we will work in the Boolean domain $\{-1,1\}^n$.

\subsection{Quick Boolean Analysis Background} We start with some necessary background. For a function $f : \{-1,1\}^n \to \reals$, we will write $\sum_{S \subseteq [n]} c_S \chi_S$ for its Fourier decomposition, where $\chi_S$ (resp. $c_S)$ is the parity function (resp. the Fourier coefficient) on $S \subseteq [n]$. 
We will say that $f$ is of degree-$d$ if $c_S = 0$ for any $|S| > d$.
Recall also that for $f : \{-1,1\}^n \to \{-1,1\}$, the influence of $f$ in direction $i \in [n]$ under $\calU(\{-1,1\}^n)$ is defined as
$$\mathrm{Inf}_i(f) = \Pr_{x \sim \calU(\{-1,1\}^n)}[f(x) \neq f(x^{\sim i})],$$ where $x^{\sim i}$ is $x$ with $x_i$ flipped. This definition can be extended when the image of $f$ is not Boolean and when the underlying measure is product but potentially not uniform.
For general $f : \{-1,1\}^n \to \reals$ and product measure $P = \otimes_{i \in [n]} P_i$ over $\{-1,1\}^n$, we have 
$\mathrm{Inf}_i(f) = \E_{x \sim P}[\Var_{x_i \sim P_i} (f(x)) ].$
In the special case where $P = \calU(\{-1,1\}^n)$, we also have
$\mathrm{Inf}_i(f) = \sum_{S \ni i} c_S^2$.
For further details on Boolean functions, see \cite{o2014analysis}.

One of our main tools in this section is the Invariance Principle of \cite{mossel2005noise}. This result is a generalization of the Berry–Esseen CLT (in fact, the invariance principle is an analogue of the Berry–Esseen theorem for low-degree polynomials) and, informally, states that the CDF of a low-degree Boolean
function with small influences can be  approximated by that of the same function when replacing the input variable by a Gaussian vector.

\begin{proposition}
[Invariance Principle \cite{mossel2005noise}]
\label{prop:invariance}
Let $X_1,...,X_n$ be independent random variables satisfying that
$\E[X_i] = 0$,
$\E[X_i^2] = 1$ and
$\E[|X_i|^3] \leq \beta$.
Let $f(x) = \sum_S c_S \chi_S$ be a degree-$d$ multilinear function such that
$\sum_{|S| > 0} c_S^2 = 1$ and
$\max_{i}\sum_{S \ni i} c_S^2 \leq \tau$. Then
\[
\sup_t
\left|
\Pr[f(X_1,...,X_n) \leq t] -
\Pr[f(Y_1,...,Y_n) \leq t]
\right| 
\leq O(d \beta^{1/3} \tau^{1/8d})\,,
\]
where $Y_i$ are independent standard Gaussians.
\end{proposition}

We will refer to the additive approximation factor $$c \cdot d \beta^{1/3} \tau^{1/8d}$$ as the \textbf{invariance principle gap}, where $c>0$ is an absolute constant.

\subsection{Main Inequality}
We proceed with our main transfer inequality for the Boolean domain.
Recall that in the Euclidean setting, the structure of the function class (polynomials of low-degree) was crucial for the argument. In the Boolean domain, \emph{any} function is essentially a polynomial in the Fourier basis.
Hence, we cannot expect that every function with low-degree Fourier spectrum will transfer. To this end, we give a sufficient condition for transfer. This condition essentially requires that the function has also small \emph{influence} in any direction.
In particular, the maximum influence of the function should be roughly upper bounded by the inverse of the mass of the seen part of the Boolean hypercube.

\begin{theorem}
[Transferability of Boolean functions]
\label{thm:transfer}
Let 
$Q$ be a product probability distribution supported on $\{-1,1\}^n$ satisfying
$\E_Q[x_i] = 0, \E_Q[x_i^2] = 1$ and $\E_Q[|x_i|^3] \leq \beta$ for any $i \in [n]$.
Let $S \subseteq \{-1,1\}^n$ be the seen part of the Boolean hypercube and
let $P = Q_S$ be the distribution $Q$ conditioned on the seen subset $S$. Let $f(x) = \sum_S c_S \chi_S$
be a degree-$d$ multilinear function 
with bounded maximum influence, i.e., $\max_{i} \sum_{S \ni i} c_S^2 \leq \tau$ and unit variance, i.e., $\sum_{S:|S|>0} c_S^2=1.$

For some universal constant $c$, assume that $Q(S)$ is larger than the invariance principle gap, i.e., 
\begin{equation}
\label{eq:small influence}
    Q(S) \geq c \cdot d \beta^{1/3} \tau^{1/8d}\,.
\end{equation} 
Then, it holds that
\[
\E_Q[f^2] \leq d^{O(d)} Q(S)^{-2d} \cdot \E_P[f^2]\,.
\]
\end{theorem}

For the proof, we refer to \Cref{app:proof2}.
In the domain $\{-1,1\}^n$, $\tau$ corresponds to the maximum influence of $f,$ i.e., $\tau = \max_i \mathrm{Inf}_i(f)$. Our condition states that $Q(S)$ should be at least of order $d \tau^{1/d}$, where $d$ is the degree of $f$. For any measure $Q$, this reveals an interesting connection between the complexity of $f$ (degree and max-$Q$-influence) and the structure of the seen part (mass under $Q$). For instance, for $Q$ being the uniform distribution over $\{-1,1\}^n$, we have that dictator functions require that $Q(S) = 1$ (anything must be seen).
We refer the reader to \Cref{section:boolean-discussion} for an extensive discussion on GOTU in Boolean domains and connections with the work of \cite{abbe2023generalization}.

\section{Applications in the Euclidean Domain}
We next employ our \Cref{theorem:main} in two application domains. In \Cref{sec:app1} we study truncated statistics and in \Cref{sec:app2-main} we study in-context transfer learning.

\subsection{Truncated Statistics}
\label{sec:truncated}
\label{sec:app1}
Our first example comes from the literature of truncated statistics \cite{galton1898examination,cohen1991truncated,daskalakis2018efficient}.
We can view learning from truncated samples as a transfer learning problem with source and target distributions defined as follows.

\subsubsection{General Transfer Learning for Truncated Gaussians} One of the most canonical examples in truncated statistics is learning Gaussians in $n$ dimensions from truncated examples. In more detail, let us consider the target distribution $Q = \calN(\mu, \Sigma)$ but assume only sample access to the source probability distribution 
$P = \calN_S(\mu, \Sigma)$, where 
\[\calN_S(\mu, \Sigma; x) = \frac{\calN(\mu, \Sigma ; x) 1\{x \in S\}}{\int_S \calN(\mu, \Sigma;y) dy}\,,
\] 
where $S \subseteq \reals^n$ corresponds to the truncation set. This is a standard example where 
the ratio $\frac{dQ(x)}{dP(x)}$ can be infinite, i.e., for $x \not \in S,$ $\frac{dQ(x)}{dP(x)} = \frac{\calN(\mu, \Sigma ; x)}{\calN_S(\mu, \Sigma; x)} = \infty$. On the other side, the opposite ratio
$\frac{d P(x)}{d Q(x)}$, on which our main inequality is based, can be easily proven to be always finite when $S$ is ``full-dimensional'', in the sense $\calN(\mu, \Sigma ; S) = \int_S \calN(\mu, \Sigma; y) dy > 0$; more precisely, it holds $\|\frac{d P}{d Q}\|_{\infty}=O(1/\calN(\mu, \Sigma ; S) ) $.

Assume sample access to a truncated normal $\calN_S(\mu^\star, \Sigma^\star)$
and let
$\calN(\mu^\star, \Sigma^\star; S) \geq \alpha$ for some $\alpha > 0$. Consider any estimator $\mu$ of $\mu^\star$ which is trained using arbitrarily many samples from $\calN_S(\mu^\star, \Sigma^\star)$. It is natural to expect that in many settings, $\mu$ can achieve a small mean-square error (MSE) performance (also known as prediction performance) against a fresh sample from the truncated Gaussian measure $\calN_S(\mu^\star, \Sigma^\star)$. 
Using \Cref{theorem:main}, we can derive that this \emph{automatically} translates to a non-trivial upper bound on the MSE against a fresh sample from the \emph{non-truncated} Gaussian measure $\calN(\mu^\star, \Sigma^\star)$.

To be more precise, the MSE of any estimate $\mu$ of $\mu^{\star}$ with respect to the truncated density is equal to
\[
\E_{y \sim \calN_S(\mu^\star, \Sigma^\star)}[ \|y -\mu\|^2]\,.
\]
Since $f(y) = \|y-\mu\|^2$ is a non-negative degree-2 polynomial in the variables of $y$ and the (non-truncated) normal distribution is a log-concave measure, we can apply our \Cref{theorem:main} to show that
\begin{align}
\E_{y \sim \calN(\mu^\star, \Sigma^\star)}[\|y- \mu\|^2]
\leq 
O(1/\alpha^2)
\cdot
\E_{y \sim \calN_S(\mu^\star, \Sigma^\star)}[\|y- \mu\|^2]\,.
\label{eq:MSE_Gaussians}
\end{align}
Additionally, using a vanilla change of measure we have that 
\begin{align}
\E_{y \sim \calN_S(\mu^\star, \Sigma^\star)}[\|y- \mu\|^2]
\leq 
(1/\alpha)
\cdot
\E_{y \sim \calN(\mu^\star, \Sigma^\star)}[\|y- \mu\|^2]\,.
\label{eq:MSE_Gaussians:2}
\end{align}
The above expression relates multiplicatively the MSE of $\mu$ with respect to the target $Q = \calN(\mu^\star, \Sigma^\star)$ to that of the source $P = \calN_S(\mu^\star, \Sigma^\star)$. The term $1/\alpha^2$ appears since we transfer a degree-2 polynomial and the density ratio, as discussed above, is $O(1/\alpha)$. We highlight that \eqref{eq:MSE_Gaussians} and \eqref{eq:MSE_Gaussians:2} hold \emph{for all truncations} $S$ with mass at least $\alpha,$ and \emph{for all measurable estimators} $\mu$.

\begin{remark}
We now discuss how our equations \eqref{eq:MSE_Gaussians}, \eqref{eq:MSE_Gaussians:2} compare with the recent relevant works on learning a Gaussian distribution in truncated statistics \cite{daskalakis2018efficient,kontonis2019efficient}. To simplify this comparison we are going to assume that $\Sigma^{\star} = I$. In \cite{daskalakis2018efficient,kontonis2019efficient} the authors construct estimators $\hat{\mu}$ from ``truncated'' samples drawn from $\calN_S(\mu^\star, I)$ that have the property that the distance $\|\mu^{\star} - \hat{\mu}\|$ is small, and in particular goes to zero as the number of samples goes to infinity, i.e., $\hat{\mu}$ is consistent. This in turn implies that the non-truncated error satisfies $$\E_{y \sim \calN(\mu^\star, I)}[ \|y - \hat \mu\|^2] \le \mathrm{OPT} + o(1),$$ where $\mathrm{OPT}$ is the minimum achievable MSE with respect to $\calN(\mu^\star, I)$. We note that the estimators from \cite{daskalakis2018efficient,kontonis2019efficient} achieving this error guarantee are quite sophisticated and carefully exploit the knowledge of the truncation $S$.

Our equations \eqref{eq:MSE_Gaussians} and \eqref{eq:MSE_Gaussians:2} show that one can control the ``non-truncated'' MSE with respect to $\calN(\mu^\star, I)$ in an immediate way: as long as one has any bound on the ``truncated'' MSE with respect to $\calN_S(\mu^\star, I)$ (so any control on the ``standard'' testing error for the samples received), the ``non-truncated'' MSE is then provably within a multiplicative factor of $O(1/\alpha^2)$ from the ``truncated'' MSE. Notably these all hold for \textit{any estimator $\mu$}.

For example, as a corollary, if we choose $\tilde{\mu}$ to minimize the MSE with respect to the truncated data this already implies a non-trivial upper bound on the non-truncated MSE; in particular we have that $\E_{y \sim \calN(\mu^\star, I)}[ \|y - \tilde \mu\|^2] \le O(1/\alpha^2) \cdot \E_{y \sim \calN_S(\mu^\star, I)}[ \|y - \tilde \mu\|^2] \leq O(1/\alpha^3) \cdot \mathrm{OPT}$ (here we used \eqref{eq:MSE_Gaussians} in the first inequality and \eqref{eq:MSE_Gaussians:2} together with the optimality of $\wt{\mu} $ in the second). This in turn can be translated to a guarantee of the form $\|\mu^{\star} - \tilde{\mu}\| \le \mathrm{poly}(1/\alpha)$. Although this estimator is not consistent, it achieves a non-trivial error \emph{without the need to know $S$ at all}. Similar estimators $\tilde{\mu}$ with $\mathrm{poly}(1/\alpha)$ error (without knowing $S)$ are already found useful as a warm start to build more sophisticated estimators, as in \cite{daskalakis2018efficient,kontonis2019efficient}, and we can upper bound the error of such $\tilde{\mu}$ by a simple application of our transfer inequality.

\end{remark}
\subsubsection{General Transfer Learning for Truncated Regression}
Our next goal is to employ the idea that we used to derive \eqref{eq:MSE_Gaussians} to bound the error of regression problems from truncated samples.

\begin{definition}
\label{def:trunc-reg}
Consider the setting where we observe
arbitrary covariates $x^{(1)},...,x^{(N)} \in \reals^n$ and labels $y_i = f^\star(x^{(i)}) + \eps_i$, where $\eps_i \sim \calN(0,1)$ for some model $f^\star : \reals^n \to \reals$ belonging to some class of functions $\calF$. 
Let $S \subseteq \reals$ be a measurable subset of the real line. 
In the truncated setting, we observe the pair $(x^{(i)}, y^{(i)})$ only if $y^{(i)} \in S$, by sampling $y^{(i)} \sim \calN_S(f^\star(x^{(i)}),1).$    
\end{definition}

\begin{remark}
    We note that the special case of  \Cref{def:trunc-reg} where $f^\star$ is a linear function $f_w^\star(x) = w^\top x$ with parameters $w \in \reals^n$ has been recently studied in \cite{daskalakis2019computationally}.
\end{remark}

Let us assume that we have sample access to a truncated regression problem, as in \Cref{def:trunc-reg}, with true model $f^\star$ and truncation set $S \subseteq \reals$. We will show that transfer learning between the truncated samples and the non-truncated samples is statistically possible under very mild assumptions. To do so, we will use our main transfer inequality to relate the truncated MSE of some model $f$, aiming to estimate $f^\star$, to the non-truncated MSE of $f$.

% First we observe that when there is no truncation (i.e., $S = \reals)$ the negative population log-likelihood objective with guess parameter $W$ reads as
% \begin{equation}
% \label{eq:no-truncation}
% \ell(W) = \frac{1}{N} \sum_{i=1}^N
% \E_{y \sim \calN(p_{W^\star}(x^{(i)}),1)}\left[
% \frac{1}{2}(y-p_W(x))^2 + \log(\sqrt{2\pi})
% \right]\,.
% \end{equation}

% Thanks to the normality of the noise random variable, the truncated population version of the negative log-likelihood with guess parameter $W$ in the above setting can be written as:
% \begin{equation}
% \label{eq:eq:full-nll}
% \ell_S(W) = \frac{1}{N} \sum_{i = 1}^N
% \E_{y \sim \calN_S(p_{W^\star}(x^{(i)}), 1)}[\l_S(W, x^{(i)}, y)]\,,
% \end{equation}
% where the negative log-likelihood of a single sample is equal to 
% \begin{equation}
%     \label{single-sample}
%     \l_S(W, x, y) = 
% \frac{1}{2}(y - p_W(x))^2
% +
% \log 
% \left(
% \int_S e^{-\frac{1}{2}(z - p_W(x))^2} dz
% \right)\,.
% \end{equation}
% Note that we recover \eqref{eq:no-truncation} by setting $S = \reals$.

Let us now rephrase the problem in the transfer learning terminology.
Consider target distribution $Q = \frac{1}{N} \sum_{i \in [N]} \calN(f^\star(x^{(i)}), 1)$ and source distribution $P = \frac{1}{N} \sum_{i \in [N]} \calN_S(f^\star(x^{(i)}), 1)$. 
The ``truncated'' MSE of a model $f$ is defined as
\[
\frac{1}{N} \sum_{i=1}^N
\E_{y \sim \calN_S(f^\star(x^{(i)}),1)} [(y - f(x^{(i)}))^2]\,
\]
and the ``non-truncated'' MSE of $f$ is similarly
\[
\frac{1}{N} \sum_{i=1}^N
\E_{y \sim \calN(f^\star(x^{(i)}),1)} [(y - f(x^{(i)}))^2]\,.
\]
Using our \Cref{theorem:main}, we can obtain the following result.

\begin{corollary} 
[Transfer in Truncated Regression]
\label{cor:truncated}
Consider the truncated regression setting $y_i = f^\star(x^{(i)}) + \eps_i, i \in [N]$ of \Cref{def:trunc-reg} with true model $f^\star$ and truncation set $S \subseteq \reals$. Let us set $\mu_i \triangleq f^\star(x^{(i)})$.
Assume that the truncation set $S$ has mass at least $\alpha > 0$ with respect to all $\mu_i, i=1,\ldots,N$ i.e., $\min_{i \in [N]}\calN(\mu_i, 1; S) \geq \alpha$.

Then, for any model $f,$ it holds that
\[
\frac{1}{N} \sum_{i=1}^N
\E_{y \sim \calN(f^\star(x^{(i)}),1)} [(y - f(x^{(i)}))^2]
\leq \left(\frac{C}{\alpha^2} \right) \cdot 
\frac{1}{N} \sum_{i=1}^N
\E_{y \sim \calN_S(f^\star(x^{(i)}),1)} [(y - f(x^{(i)}))^2]\,,
\]
for some absolute constant $C > 0.$
\end{corollary}

The above inequality allows us to
relate multiplicatively the ``truncated'' MS' of an arbitrary prediction model $f$ of the true model $f^\star$ to the MSE of $f$ over the whole non-truncated measure. 

For the proof, we refer to \Cref{app:truncated}.
We mention that \cite{daskalakis2019computationally} established a similar inequality to show that transfer learning is possible under the assumption that $f^\star$ is linear\footnote{Also, the assumption that the set $S$ has mass with respect to any Gaussian $\calN(\mu_i,1)$ can be weakened under some structural properties for the family $\calF$. For instance, for linear functions it suffices to consider a lower bound on the average of the masses (see Assumption 1 in \cite{daskalakis2019computationally}).}. The above result applies much more generally and, in particular, note that the exact details of the function $f^\star$ do not play a role: if $f^\star$ is linear or a degree-100 polynomial, the ``transfer coefficient'' is still simply $C/\alpha^2.$

\begin{remark}
We should mention that the above result, does not discuss the (non-trivial) details on how to efficiently find an estimator that has in fact small MSE under the truncated measure. While this is beyond the scope of this work, we would like to mention that for many natural families of parameterized regression functions $\calF = \{f_w : w \in \calW\}$, the truncated log-likelihood is concave with respect to $w$ and, under some mild assumptions, one can design a first-order method to optimize it efficiently and obtain a small truncated MSE. For instance, such an analysis is provided by \cite{daskalakis2019computationally} for the case of linear functions $f_w(x) = w^\top x$.
\end{remark}

\subsection{In-Context Transfer Learning}
\label{sec:app2-main}

Transformers \cite{vaswani2017attention} and in-context learning (ICL) have been extensively explored and studied recently \cite{von2023transformers,akyurek2022learning,garg2022can,li2023transformers,bai2023transformers,ahn2023transformers,schlag2021linear}, mainly because of their applications in text generation via large language models (LLMs). 
In-context learning \cite{garg2022can} corresponds to the ability of a model to condition on a prompt sequence consisting of
in-context examples (input-output pairs corresponding to some task) along with a new query input, and
generate the corresponding output. 
\subsubsection{Background on ICL}
We give the formal setting below.

\begin{definition}
[Class and Distributions]
We consider a feature space $\calX$ and a label space $\calY$. We let $H \subseteq \calY^\calX$ be a concept class of functions $\calX \to 
\calY$. We let $P_X$ be a distribution over the feature space $\calX$
and $P_H$ be a distribution over functions in $H$.
\label{def:setup}
\end{definition}

For our application, we will consider $H$ to be the set of linear functions $\{h(x) = w^\top x : w \in \reals^n\}$. As an example, $P_X$ could be the standard normal distribution and $P_H$ could generate a random linear function $x \mapsto w^\top x$ by drawing $w$ from some probability measure over $\reals^n$.

\begin{definition}
[Training/Testing In-Context Samples]
A realizable training sample or \emph{training prompt} of length $M \in \nats$ from $(P_X,P_H)$ corresponds to a random sequence $(x_1,y_1,...,x_M,y_M) \in (\calX \times \calY)^M$ generated as follows.
\begin{enumerate}
    \item We first draw a concept $h \sim P_H$.
    \item Independently of $h$, we draw $x_1,...x_M$ i.i.d. from $P_X$.
    \item We set $y_i = h(x_i)$ for any $i \in [M]$ and return $(x_1,y_1,...,x_M,y_M)$.
\end{enumerate}
A realizable testing sample or \emph{testing prompt} of length $N \in \nats$ from $(P_X,P_H)$ corresponds to a random sequence $(x_1,y_1,...,x_N,y_N, x_{\mathrm{query}}) \in (\calX \times \calY)^N \times \calX$ generated as follows.
\begin{enumerate}
    \item We first draw $h \sim P_H$.
    \item Independently of $h$, we draw $x_1,...x_N, x_{\mathrm{query}}$ i.i.d. from $P_X$.
    \item We set $y_i = h(x_i)$ for any $i \in [N]$ and return $(x_1,y_1,...,x_N,y_N, x_{\mathrm{query}})$.
\end{enumerate}
\label{def:icl-train-test}
\end{definition}
Note that in each training prompt the labeling hypothesis will likely be different since it is drawn from $P_H$; this is an important difference from the standard PAC setting.

In what follows, we will need the following notation:
\begin{enumerate}
    \item Fix $H \subseteq \calY^\calX$ , $P_X$ and $P_H$, as in \Cref{def:setup}.
    \item 
Let $\l : \calY \times \calY \to \reals$ be a loss function, e.g., the squared loss.
\item 
Let $\calS = \bigcup_{m \in \nats} \{(x_1,y_1,...,x_m,y_m), x_i \in \calX, y_i \in \calY\}$ be the set of finite-length sequences of $(x,y)$ pairs.
\end{enumerate}

Having trained a model with realizable training prompts, each of length $M$, we test its performance on new prompts of length $N$, potentially different than the training length $M$. The learning guarantee of such a model is formally defined as follows.

% The goal of in-context learning is to train a model with training prompts and, given a test prompt, the model must predict the label of the query point $x_{\mathrm{query}}$ and incur some loss.
% Typically we can define the training phase of in-context learning on prompts of length $M$ as follows:
% \begin{definition}
% [In-Context Training \cite{zhang2023trained}]
% \label{def:icl-train}
% Let $\calF_\Theta = \{ f_\theta : \calS \times \calX \to \calY , \theta \in \Theta \}$ be a class of functions parameterized by $\theta$ in some parameter space $\Theta$. For $M > 0$, we say that a model $f : \calS \times \calX \to \calY$ is trained on realizable training prompts of length $M$ from $(P_X,P_H)$ under loss $\ell$ if 
% \begin{equation}
% \label{eq:icl-train-loss}
% f = f_{\theta^\star} 
% ~~~\textnormal{with}~~~
% \theta^\star = \argmin_{\theta \in \Theta}
% \E_{\calT} [\ell(f_\theta(\calT), h(x_M))]\,,    
% \end{equation}
% where $\calT = (x_1,h(x_1),...,x_{M-1},h(x_{M-1}), x_M)$ with $x_i \sim P_X, i \in [M]$ and $h \sim P_H$ independently.
% \end{definition}
% Note that the definition of the model $f$ as a mapping from $\calS \times \calX \to \calY$ allows us to have the flexibility to change the length of the given prompt.
% Having trained the model by running, for instance, stochastic gradient descent on the loss objective of \eqref{eq:icl-train-loss}, we test its performance on new prompts of length $N$, potentially different than the training length $M$. 

\begin{definition}
[In-Context Learning \cite{zhang2023trained}]
\label{def:icl learn}
For $M, N > 0,$ we say that a model 
$f : \calS \times \calX \to \calY$, trained on realizable prompts of the form
$(x_{1}, y_1,...,x_{M}, y_{M})$, 
\emph{in-context learns a concept class $H$ under loss $\l$ 
in the realizable setting
with respect to
$(P_X, P_H)$ up to error $\eta \geq 0$}
if there exists a function $N_{ P_X, P_H}: (0,1) \to \nats$
such that for every $\eps \in (0,1)$, if the length of the testing prompt $N \geq N_{P_X, P_H}(\eps)$, then
\[
\E_{\calT}
[\l(f(\calT), h(x_{\mathrm{query}}))] \leq \eta + \eps\,,
\]
where $\calT = (x_1, h(x_1),...,x_N, h(x_N), x_{\mathrm{query}})$
with $x_1,...,x_N,x_{\mathrm{query}} \sim P_X$ and $h \sim  P_H$ independently.
% \footnote{ This definition does not explain how to train a model from in-context examples. 
% The natural training algorithm is to minimize the empirical loss which works as follows:
% we 
% assume a model class $\calF_\Theta = \{f_\theta : \calS \times \calX \to \calY : \theta \in \Theta\}$.
% First, we sample independent prompts by sampling a random function $h \sim P_H$ and feature vectors $x_1,...,x_N,x_{\mathrm{query}}$
% i.i.d. from $P_X$, and then minimize the empirical version of the population loss objective using stochastic gradient
% descent or other stochastic optimization algorithms on the parameter space $\Theta$ \cite{zhang2023trained} (the $t$-th training prompt is of the form $(x_{t,1}, h_t(x_{t,1}),...,x_{t,N}, h_t(x_{t,N}))$).
% In particular, given the $t$-th prompt during training, we ask the model to predict each label $h_t(x_{t,i}), i \leq N$ given the prompt $x_{t,1},h_t(x_{t,1}),...,x_{t,i-1}, h_t(x_{t,i-1}),x_{t,i}$ with respect to the squared loss \cite{garg2022can}.
% }
\end{definition}

We remark that the definition of the model $f$ as a mapping from $\calS \times \calX \to \calY$ allows us to have the flexibility to change the length of the given prompt.
Essentially, the length of the testing prompt $N$ controls the error parameter $\eps$, while the length of the training prompt $M$ controls the error parameter $\eta$ (which can be non-zero even in the realizable setting). 
Note also that by taking $P_H$ to be point mass and $N=0$, we recover the setting of realizable PAC regression under loss $\ell$ if the 
model learns under any $P_X$ and any point mass distribution $P_H$
\cite{attias2024optimal}.

In the next section, we specialize the above general abstract setup to the concrete problem of in-context learning linear functions with single-layer linear self-attention functions.

\subsubsection{Problem Setup and Motivation: the Linear Case}
{
Let us consider the problem of in-context learning the concept class $H = \{ h(x) = w^\top x : w \in \reals^n \}$ of linear functions in $n$ dimensions \cite{garg2022can,zhang2023trained} in the realizable setting (i.e., $y_i = w^\top x_i)$ with loss function  $\ell(y,y') = (y-y')^2$ under distributions $(P_X,P_H)$; 
$P_H$ generates a random linear function $x \mapsto w^\top x$ by drawing $w \sim P_H$.

%For instance, the standard choice in prior works \cite{garg2022can,zhang2023trained} is to take $P_X = P_H = \calN(0,I).$
%The input to the model is a prompt (of size $N)$, i.e., a sequence $(x_1,h(x_1),...,x_N,h(x_N),x_{\mathrm{query}})$, where inputs (i.e., $x_i$ and $x_{\mathrm{query}})$ are drawn i.i.d. from $P_X$ and $h(x) = w^\top x$, where $w$ is drawn from $P_{H}$. 
% The task of in-context learning linear functions consists of two steps: 
% \begin{enumerate}
%     \item based on \Cref{def:icl-train}, we train a model using independent random  prompts; each prompt is labeled with a different random linear function drawn from $P_H$ and has length $M$.
%     \item based on \Cref{def:icl learn}, without changing the parameters of the trained model, we test its generalization abilities on new 
% unseen prompts of length $N$ from the same pair $(P_X,P_H)$, i.e., when given a realizable testing prompt $(x_1, w^\top x_1,...,x_N, w^\top x_N,x_{\mathrm{query}})$, the goal is to predict the label of $x_{\mathrm{query}}$. 
% \end{enumerate}

It remains to specify the model that we will train in order to in-context learn linear functions. The standard model of empirical studies is the so-called single-head attention layer defined as $f_{\mathrm{SA}}(X) = E + W^{PV} E \cdot  \mathrm{softmax} \left( \frac{E^\top W^{KQ} E}{\rho}\right)$, where $E$ is the input prompt, $W^{PV}, W^{KQ}$ are trainable matrices and $\rho > 0$ is a normalization factor. In our theoretical application, we follow the route of recent works by \cite{zhang2023trained,von2023transformers,li2023transformers,ahn2023transformers,suvrit2023} and study single-head \emph{linear} attention layers $f_{\mathrm{LSA}}$, a model which essentially removes the softmax operator from $f_{\mathrm{SA}}$. Following \cite{zhang2023trained}, we can formally define the single-layer linear self-attention (LSA) model with trainable parameters $\theta = (W^{KQ}, W^{PV})$ and input data matrix \[
E = \begin{bmatrix}
x_1 & ... & x_N & x_{\mathrm{query}}\\
y_1 & ... & y_N & 0
\end{bmatrix} \in \reals^{(n+1) \times (N+1)}\] 
as
\begin{equation}
\label{eq:lsa}
    f_{\mathrm{LSA}}(E; \theta) = 
    E + W^{PV} E \cdot \frac{E^\top W^{KQ} E}{\rho}\,,
\end{equation}
where $\rho > 0$ is some normalization factor (we take it equal to $N$).
We note that $f_{\mathrm{LSA}}$ is a matrix and the model's prediction to some prompt of length $N$ is
\begin{equation}
\wh{y}_{\mathrm{query}}(E; \theta) = f_{\mathrm{LSA}}(E;\theta)[n+1, N+1]\,,  
\label{eq:prediction}  
\end{equation}
i.e., the bottom-right entry of the matrix, where $n$ is the dimension of the ambient space.

Considering some realizable testing prompt $E$, which essentially corresponds to the sequence obtained by the independent random variables $w \sim P_H$, $x_{i}, x_{\mathrm{query}} \sim P_X$ for $i \in [N]$, the population loss of the LSA model with parameters $\theta$ is
\begin{equation}
\label{eq:pop-loss}
L_P(\theta) = 
\E_{P_X^{\otimes N+1}, P_H} [(\wh{y}_{\mathrm{query}}(E;\theta) - w^\top x_{\mathrm{query}})^2]\,.
\end{equation}

In the above and hereafter, we will use the notation $P_X^{\otimes N+1}, P_H$ to denote the distribution of a single test prompt, i.e., we will abbreviate the notation $\E_{x_1,...,x_N,x_{\mathrm{query}},w}[\cdot]$ with $\E_{P_X^{\otimes N+1}, P_H}[\cdot].$

{

\subsubsection{Application of Transfer Learning to ICL Linear Functions.}
We now turn our attention to our problem of interest.
We consider the question of \emph{in-context transfer learning linear functions} using the model $f_{\mathrm{LSA}}$. In this setting, the model is given sample access to prompts generated by the source pair $(P_X,P_H)$ but we want to test its generalization abilities on potentially different target distributions $(Q_X,Q_H).$ 
From a theoretical front, this question was also studied by \cite{zhang2023trained} where they considered the robustness of the trained linear attention model under distribution shifts when $P_X = \calN(0, \Sigma)$ for some covariance matrix $\Sigma$. We remark that under this choice of source distribution the prediction of the model $f_{\mathrm{LSA}}$ has a nice closed form at the global optimum\footnote{
\cite{zhang2023trained} train 
an LSA model $f_{\mathrm{LSA}}(\cdot;\theta)$ using gradient flow 
with training prompts of length $M$ and show that gradient flow converges 
to some parameters $\theta^\star$ (see Theorem 4.1 of \cite{zhang2023trained}).
They then check the generalization properties of the model $f_{\mathrm{LSA}}(\cdot;\theta^\star)$.
If $P_X = \calN(0,\Sigma)$, then the prediction 
of the LSA model at $\theta^\star$, trained on prompts of length $M$ when given a testing prompt of length $N$ is
$\wh{y}_{\mathrm{query}} \approx x^\top_{\mathrm{query}} \Sigma^{-1} \left( \frac{1}{N}\sum_{i=1}^N x_i y_i \right)$ for input $(x_1,y_1,...,x_N,y_N,x_{\mathrm{query}}$) when the training length $M$ is large. This closed form expression allows the authors to analyze task,covariate and query distribution shifts on the way they generate the testing prompt. 
}, which allowed the authors to study task, covariate and query distribution shifts $(Q_X,Q_H)$.

In our application of  \Cref{theorem:main}, we allow for general source and target distributions.
As such, our application gives the first general rigorous transfer learning result for in-context learning in this setting allowing us to study distribution shifts in a more systematic manner. 
In what follows we fix the length of the input prompt to be $N$.

 In \Cref{thm:icl transfer}, we give a general transfer result applied on the setting of in-context learning linear functions with single-layer linear self-attention layer models in a slightly more general setting than the one we described before:
we consider three \emph{arbitrary} source distributions: $P_X$ over input features, $P_X^{\mathrm{query}}$ over the query feature point and $P_H$ over normal vectors. In a similar manner, we define $Q_X,Q_X^{\mathrm{query}},Q_H$ for the \emph{arbitrary} associated shifts.
\begin{corollary}
[In-Context Transfer Learning Linear Functions with LSA Models]
\label{thm:icl transfer}
\label{cor:icl}
Consider the setting of in-context learning linear functions with a single linear self-attention layer $f_{\mathrm{LSA}}$ of \eqref{eq:lsa} with prompts of length $N$ in $n$ dimensions. 

Consider product distributions $Q = (Q_X^{\otimes N}, Q_X^{\mathrm{query}}, Q_H)$ (target) and $P = (P_X^{\otimes N}, P_X^{\mathrm{query}}, P_H)$ (source) with associated population loss functions $L_Q$ and $L_P$ respectively, defined as in \eqref{eq:pop-loss} with (fixed) parameters $\theta$. 
Let $\calL_m$ be the space of $m$-dimensional log-concave distributions.
Then the following hold true. 
\begin{enumerate}
\item It holds that
\[
L_Q(\theta) \leq C \cdot
\inf_{\mu \in \calL_{n(N+2)}} \| d Q/\mu \|_{\infty} \| d P/\mu \|^{c}_{\infty} \cdot 
L_P(\theta)\,,
\]
for some positive absolute constants $C,c$ with $c \in \nats$.
\item Assume that $P_X, P_X^{\mathrm{query}}$ are log-concave. The cost of task shift $P_H \to Q_H$ (i.e., $P_X = Q_X$ and $P_X^{\mathrm{query}} = Q_X^{\mathrm{query}}$) incurred by the model is
\[
L_Q(\theta) \leq 
C \cdot 
\inf_{\mu \in \calL_n} \| dQ_H/d\mu \|_{\infty} \| dP_H/d\mu \|^{c}_{\infty} \cdot L_P(\theta)\,,
\]
for some positive absolute constants $C,c$ with $c \in \nats$.
    \item Assume that $P_X, P_H$ are log-concave. The cost of query shift $P_X^{\mathrm{query}} \to Q_X^{\mathrm{query}}$ (i.e., $P_X = Q_X$ and $P_H = Q_H$) incurred by the model is
\[
L_Q(\theta) \leq 
C \cdot 
\inf_{\mu \in \calL_n} \| dQ_X^{\mathrm{query}}/d\mu \|_{\infty} \| dP_X^{\mathrm{query}}/d\mu \|^{c}_{\infty} \cdot L_P(\theta)\,,
\]
for some positive absolute constants $C,c$ with $c \in \nats$.
\end{enumerate}  
\end{corollary}

For the proof, we refer to \Cref{app:proof3}.
We believe that our result provides a more abstract way to study such distribution shifts and adds further insight on the impact of out-of-distribution generalization studied in \cite{zhang2023trained}.

%  \begin{corollary}
% [Informal; see \Cref{cor:icl}]
%     \label{cor:icl-main}
%    Consider the setting of in-context learning linear functions with a single-head linear attention layer $f_{\mathrm{LSA}}$ with prompts of length $N$ in $n$ dimensions and product distributions $Q = (Q_X^{\otimes N+1}, Q_H)$ (target) and $P = (P_X^{\otimes N+1}, P_H)$ (source) with associated population loss functions $L_Q$ and $L_P$ respectively, defined as in \eqref{equation:lsa} with parameters $W = (W_1,W_2)$. 
% Let $\calL_m$ be the space of $m$-dimensional log-concave distributions.
% Then there exists an absolute constant $C$ such that
% \[
% L_Q(W) \leq C \cdot
% \min_{\mu \in \calL_{n(N+2)}} \| d Q/\mu \|_{\infty} \| d P/\mu \|^{O(1)}_{\infty} \cdot L_P(W).
% \]
%  \end{corollary}

% We focus on \textbf{in-context transfer learning}. In particular, we would like to understand whether models that are able to in-context learn a class $H$ with respect to $(P_X, P_H)$, as in \Cref{def:icl learn}, are able to perform transfer learning under different distributions $(Q_X, Q_H).$

\subsubsection{Technical Details}
At a technical level, we leverage one of the main observations in \cite{zhang2023trained} to simplify the expression of \eqref{eq:pop-loss}. In particular,
it can be shown that the prediction
$\wh{y}_{\mathrm{query}}(E;\theta)$ on input $E$
can be expressed as a quadratic function $u^\top H u$
for some matrix $H$ depending only on the
input matrix $E$ and for some vector $u$ depending only on the parameters $\theta$.

\begin{lemma}
[Lemma 5.1 in \cite{zhang2023trained}]
Consider the distribution  $P = (P_X^{\otimes N+1}, P_H)$
over the input prompt, where $x_1,...,x_N,x_{\mathrm{query}} \sim P_X$ and $w \sim P_H$ are independent random variables and $y_i = w^\top x_i$.
Consider the LSA model with parameters $\theta$ and input matrix $E = \begin{bmatrix}
x_1 & ... & x_N & x_{\mathrm{query}}\\
y_1 & ... & y_N & 0
\end{bmatrix}$. Then the population loss under $P$ (see \eqref{eq:pop-loss}) can be written as
\begin{equation}
\label{eq:loss}
L_P(u) = \E_{P_X^{\otimes N+1}, P_H} [(u^\top H u - w^\top x_{\mathrm{query}})^2]\,.
\end{equation}
In the above, we have that
\[
H = \frac{X}{2} \otimes \frac{E E^\top}{N} \in \reals^{(n+1)^2 \times (n+1)^2}, ~~~~~~ X = 
\begin{bmatrix}
0_{n \times n} & x_{\mathrm{query}} \\
x_{\mathrm{query}}^\top & 0
\end{bmatrix} \in \reals^{(n+1) \times (n+1)}\,,
\]
where $\otimes$ corresponds to the Kronecker product operator
and $u \in \reals^{(n+1)^2}$ depends only on the parameters $\theta$.
\label{lemma:quadratic}
\end{lemma}

%Inspecting \eqref{eq:loss}, we can see that $L_P(u)$ is a mixture of degree-4 polynomials with respect to $u$ and is a non-convex function of $u$. 

Our goal is to show a general transfer inequality for in context learning linear functions with a single linear attention layer, i.e., we would like to relate $L_P(u) = \E_{P_X^{\otimes N+1}, P_H} [(u^\top Hu - w^\top x_{\mathrm{query}})^2]$ with $L_Q(u) = \E_{Q_X^{\otimes N+1}, Q_H} [(u^\top Hu - w^\top x_{\mathrm{query}})^2]$ for some other distributions $Q_X$ and $Q_H.$

To this end, let us fix the parameters $u$ and consider the polynomial on $H,x,w$:
\begin{equation}
\label{eq:polynomial-lsa}
f(H, x, w) = (u^\top H  u - w^\top x)^2\,.    
\end{equation}
It follows from \Cref{lemma:quadratic} and specifically the definitions of $H$ and $u$ that
$f : \reals^{n(N+2)} \to \reals_{\geq 0}$ is actually a non-negative polynomial of degree-10\footnote{By inspecting the structure of $H$ in \Cref{lemma:quadratic}, we can see that each entry of $E E^\top$ is a degree-4 polynomial in the variables $x_1,...,x_N,x_{\mathrm{query}},w$, since $y_i = w^\top x_i$ is degree-2. Then, applying the Kronecker product $X \otimes EE^\top$ increases the degree of each entry by at most 1 due to the structure of $X$. Hence, since we consider the square loss, the total degree of $(u^\top H u - w^\top x_{\mathrm{query}})^2$ is 10.} on the variables $x_{1} = (x_{11},...,x_{1n})^\top,...,x_{N} = (x_{N1},...,x_{Nn})^\top, x_{\mathrm{query}} = (x_{\mathrm{query},1},...,x_{\mathrm{query},n})^\top$ and $w = (w_1,...,w_n)^\top$. This allows us to relate $\E_{P_X^{\otimes N+1}, P_H}[f]$ to
$\E_{Q_X^{\otimes N+1}, Q_H}[f]$ in a straightforward manner given \Cref{theorem:main}. For the details, we refer to \Cref{app:proof3}.

\subsubsection{Further Comparison With \cite{zhang2023trained}}
\cite{zhang2023trained} consider the above problem of in-context learning linear functions with the LSA model when $P_X =\calN(0,\Sigma)$ and $P_H = \calN(0,I_n)$ and prove the following (among others).
\begin{enumerate}
    \item In terms of ICL training, using prompts of length $M$, 
    they manage to prove that
    gradient flow on the population loss objective for 
    LSA models and linear targets, i.e., the function 
    \[
    \theta \mapsto \E_{P_X^{\otimes M+1}, P_H} [(\wh{y}_{\mathrm{query}}(E;\theta) - w^\top x_{\mathrm{query}})^2]\,,
    \]
    converges to a global optimum $\theta^\star = (W^{KQ}_\star, W^{PV}_\star)$ under some appropriate initialization conditions. Note that the population objective corresponds to the case where the number of training prompts goes to infinity.
    \item In terms of ICL prediction error, \cite{zhang2023trained} study the loss of the LSA model in ICL linear functions in $n$ dimensions  with prompts of length $N$ (see \eqref{eq:pop-loss}). They show that the model with parameters $\theta^\star$, obtained by Item (1) under $P_X = \calN(0,\Sigma)$, satisfies
    \begin{equation}
    L_P(\theta^\star) \leq \eta + \eps\,,
    \label{eq:loss-testing}    
    \end{equation}
    where
    $\eta$ decreases with the training length $M$ as
    \[
    \eta = \frac{1 + 2n + n^2 \kappa}{M^2} \cdot \mathrm{Tr}(\Sigma)\,,
    \]
    and 
    $\eps$ decreases with the test length $N$ as
    \[
    \eps = \frac{n+1}{N} \cdot \mathrm{Tr}(\Sigma)\,,
    \]
    where $\mathrm{Tr}$ is the trace operator and $\kappa$ is the condition number of $\Sigma$. Interestingly, as proved by \cite{zhang2023trained}, the linear transformer $f_{\mathrm{LSA}}(\cdot;\theta^\star)$ learns the class of linear functions up to
error $\eta $, which scales as $O(1/M^2)$ (recall \Cref{def:icl learn}), even if the test length $N \to \infty$. Hence, training on finite-length prompts leads to a
prediction error which is, in general, bounded away from zero.
\end{enumerate}

To summarize, \cite{zhang2023trained} establish that under suitable distributional (i.e., Gaussianity) and initialization conditions, gradient-flow, trained on prompts of length $M$, converges to some parameters $\theta^\star$. More to that they show that the model $f_{\mathrm{LSA}}(\cdot;\theta^\star)$ achieves 
    small generalization loss when tested on prompts of length $N$, as we described in \eqref{eq:loss-testing}.
    Hence, our transfer results in \Cref{thm:icl transfer} can then be directly used to extend \eqref{eq:loss-testing} to general distributions $Q_X$ and $Q_H$ under the cost of the ``transfer coefficient''.

\section{Applications in the Boolean Domain: Generalization on the Unseen}
\label{section:boolean-discussion}

In this section, we discuss an application of \Cref{thm:transfer} in the Generalization on the Unseen (GOTU) setting, introduced by \cite{abbe2023generalization}. 

Let us first formally introduce the GOTU setup of \cite{abbe2023generalization}.
Let us consider $\Omega = \{-1,1\}^n$ and $S \subseteq \Omega$ be the seen part.
In this setting, there is some distribution $\nu$ over the Boolean hypercube and we only observe samples from the conditional distribution on $S$, i.e., we train our models from samples drawn from $\nu_S$ using some gradient-based method and aim to 
see how well the model generalizes on the unseen part, i.e., we test the model on $\nu.$

For simplicity, following \cite{abbe2023generalization}, we can think of the underlying measure as the uniform distribution $\nu = \calU \triangleq \mathcal{U}(\Omega)$ over the Boolean hypercube.

\subsection{A General Transfer Inequality from Seen to Unseen}

Let $P = \calU_S$ be uniformly supported on the seen part $S$. It is natural to translate GOTU in terms of transfer learning: one simply wants to learn using samples from the source $P$ and then achieve good performance on the target $Q=\calU$.

%However, 
%we can use the uniform distribution
%$\mu = \calU$ over $\Omega$ to obtain for all %$f^2$:
%\[
%\E_{\calU_{S^c}}[f^2]=\E_Q[f^2] \leq \left\|\frac{dQ}{d \calU}\right\|_\infty \E_{\calU}[f^2] = \left\|\frac{\calU(x)/\calU(S^c)}{\calU(x)}\right\|_\infty \E_{\calU}[f^2] = %\frac{1}{\calU(S^c)} \E_{\calU}[f^2]\,. 
%\]

Our \Cref{thm:transfer} directly implies a general transfer inequality between $ P$ and  $Q$ as follows. Note that $\|dQ / dP\|_\infty$ is infinite but $\|dP / dQ\|_\infty=\frac{1}{\calU(S)}$. Hence, we can use \Cref{thm:transfer} (since the invariance principle is applicable on $Q$) and obtain for all $f$ satisfying the conditions of \Cref{thm:transfer} (i.e., $\leq d$-degree and ``low'' total influence) that
\begin{align}\label{eq:boolean_tr_0}
 \E_{\calU}[f^2] \leq \frac{1}{\calU(S)^{O(d)} }\E_{\calU_S}[f^2]   
\end{align}
One may wonder how \eqref{eq:boolean_tr_0} relates with the GOTU literature. We investigate this in the following section.

\subsection{Canonical Holdout with Diagonal Linear Networks}
Let us focus on a specific GOTU setup for training diagonal linear networks, 
as in \cite{abbe2023generalization}. 
To introduce the setup, let us consider the model of diagonal linear neural networks \cite{pesme2021implicit,even2023s}.
\begin{definition}
\label{def:dll}
    We define a \emph{diagonal linear network with bias} with parameter vector
    \[
    \theta = (b, w_1^{(1)},...,w_{n}^{(1)},...,w_{1}^{(L)},...,w_{n}^{(L)})\,,
    \]
    as
    \[
    f_{\mathrm{NN}}(x;\theta)
    =
    b
    + 
    \sum_{i \in [n]}
    \left( \prod_{\l = 1}^L w_i^{(\l)} \right) x_i\,,
    \]
    where $n$ corresponds to the input dimension and $L$ is the network's depth. 

    The \textit{target function} that we want to learn in this section is a \textit{linear} function which we denote as $f^\star : \{-1 , 1\}^n \to \mathbb{R}$.
\end{definition}

We will work in
the ``canonical holdout'' setting where we freeze one coordinate of the Boolean hypercube during training, i.e.,
the seen set is $S = \{x \in \{-1,1\}^n | x_k = 1\}$ for some $k \in [n]$. Given samples from the seen set, \cite{abbe2023generalization} train 
a diagonal linear network using gradient flow to fit the unknown linear target function $f^\star$.
Since we train the diagonal linear network of \Cref{def:dll} with gradient flow, we define the trainable parameters for $t>0$ as
\begin{equation}
    \label{eq:theta-t}
\theta_t \triangleq
(b(t), w_1^{(1)}(t),...,w_{n}^{(1)}(t),...,w_{1}^{(L)}(t),...,w_{n}^{(L)}(t))\,.
\end{equation}

Let $\calU$ be the uniform distribution over $\{-1,1\}^n$ and let $\calU_S$ be the source distribution which is uniform in the set $S$ defined above (the ``canonical holdout'' setting).
We study the question of GOTU, in the following precise sense; we define
\begin{itemize}
    \item[(1)] the error in the seen data at time $t$ as
\begin{align}\label{eq:risk_1}
 L_S(\theta_t) =
\E_{x \sim \calU_S}[(f_{\mathrm{NN}}(x; \theta_t) - f^\star(x))^2]   
\end{align}where $f_{\mathrm{NN}}$ is the diagonal linear network with parameters $\theta_t$ and $f^{\star}$ is the true target linear function;
\item[(2)] the generalization on the unseen error at time $t$, as the error in the whole population at time $t$
\begin{align}\label{eq:risk_2}
    L(\theta_t) = 
\E_{x \sim \calU}[(f_{\mathrm{NN}}(x; \theta_t) - f^\star(x))^2]\,.
\end{align}
%We remark that we can take the expectation over the whole domain and not only the unseen part $S^c$ since the model is overparameterized and $L_S(\theta) \to 0.$
\end{itemize}

%Our goal is to relate these two risks using our \Cref{thm:transfer}.

This diagonal linear network is clearly \emph{overparameterized} for the task of learning linear functions when the number of layers $L \geq 2$.
Indeed, in \cite{abbe2023generalization} it is shown that the diagonal linear network, when trained using \emph{gradient flow}, will ``memorize'' the whole dataset (since it is labeled by a linear function) and hence can achieve error $o(1)$ in the training loss on the seen part, i.e., $L_S(\theta_t)$ converges to $0$ as the time $t$ grows. However, as mentioned by \cite{abbe2023generalization}, the generalization on the unseen error 
$L(\theta_t)$
will not be tending to zero since $x_k=1$ throughout training; they prove that it will eventually be of order $\mathrm{Inf}_k(f^\star) = \left(\wh{f}^\star(\{k\})\right)^2.$ In particular, ``transfer'' from $\calU_S$ to $\calU$ is \emph{not} happening for the function $x \mapsto (f_{\mathrm{NN}}(x; \theta_t) - f^\star(x))^2$ for large enough $t$.

Based on our general transfer inequality \eqref{eq:boolean_tr_0} we can immediately conclude that, for large enough $t$, the low-degree function $x \mapsto (f_{\mathrm{NN}}(x; \theta_t) - f^\star(x))^2$ does not satisfy the conditions of \Cref{thm:transfer}. Motivated by a large body of work on ``early stopping'' \cite{yao2007early,li2020gradient,ji2021early}, it is natural to ask whether there is hope that transfer does hold for earlier stages in training. 

Indeed, in the next result, we show that the conditions of \Cref{thm:transfer} do hold for the function $x \mapsto (f_{\mathrm{NN}}(x; \theta_t) - f^\star(x))^2$ for some non-trivial initial  interval of time during training. To be more specific, recall that we need the function $f_{\mathrm{NN}}(\cdot~; \theta_t) - f^\star(\cdot)$ (which is by definition linear, i.e., a $d=1$-degree polynomial on $x$) to be of sufficiently small maximum influence, as described in \Cref{eq:small influence}. A quick calculation gives that its maximum influence $\tau(t)$ is given by
\[
\tau(t) = \frac{\max_{i \in [n]} \left( \prod_{\l = 1}^L w_i^{(\l)}(t) - \wh{f}(\{i\}) \right)^2 }
{\sum_{i \in [n]} \left( \prod_{\l = 1}^L w_i^{(\l)}(t) - \wh{f}(\{i\}) \right)^2}\,.
\] 
Hence for our transfer inequality of \Cref{thm:transfer} to be applicable, we need to investigate whether $$\tau(t) \leq c (Q(S)/d)^{8d},~ \textnormal{for some absolute constant} ~c\,,$$ which since, in our case, $d=1$ and $Q(S)=1/2$ boils down to $$\tau(t) < c_0,~\textnormal{for some absolute constant} ~c_0\,.$$

Our result shows that, with high probability over the random initialization of the weights of the diagonal linear network, there exists a time $t^\star$ such that $\tau(t) < c_0$ for $t < t^\star$ (and so transfer happens) and $\tau(t) > c_0$ for $t > t^\star$ (our inequality is not applicable).

%The maximum influence at time $t$ is controlled as follows: 
% first, {it follows from the analysis of \cite{abbe2023generalization} that, for any $t$, we have a small Fourier coefficient in the unseen direction $\wh{f}_{\mathrm{NN}}(\{k\}) \leq \eps$. This can be used to show that the gap}
% $(\wh{f}_{\mathrm{NN}}(\{k\}) - \wh{f}(\{k\}))^2$ is always large. On the other side, for any other direction $i \neq k$, gradient flow gives that 
% $(\wh{f}_{\mathrm{NN}}(\{i\})(t) - \wh{f}(\{i\}) )^2 
% \leq 
% (\wh{f}_{\mathrm{NN}}(\{i\})(0) - \wh{f}(\{i\}) )^2 
% O(e^{-t})\,.
% $ 
% Hence, after the random initialization, the gaps 
% for any $i \neq k$ shrink as $t$ grows while the gap in the unseen direction $k$ remains large. As a result, as $t$ grows, $\tau(t) \to 1$ and there is a $t^\star$ where our transfer condition stops to hold and essentially the NN has achieved almost $0$ training loss (which cannot be transferred to the unseen part). This observation is also depicted in \Cref{fig:dln1} and \Cref{fig:dln2}.

\begin{proposition}
\label{thm:DLN-transfer}
Let $n$ be sufficiently large.
Fix $k \in [n].$
Let $f^\star : \{-1,1\}^n \to \reals$ be an arbitrary linear function with $|\wh{f}^\star(\{k\})| > \gamma > 0$.
Consider learning $f^\star$ with gradient flow on a diagonal linear network $f_{\mathrm{NN}}(\cdot~; \theta_t)$ with parameters $\theta_t$ as in \eqref{eq:theta-t} and depth $L \geq 2$. The source distribution is the uniform over
the seen part $S = \{ x \in \{-1,1\}^n | x_k = 1\}$.
For any $\eps \in  (0,\gamma/2)$, 
there exists an $\alpha_{\max}$ (increasing with $L$, depending on $\eps$) such that
if all the model's parameters $\theta_0$ at time $t = 0$ are initialized i.i.d.
under the uniform distribution on $(-\alpha, \alpha)$ for any $0 < \alpha \leq \min\{\alpha_{\max}, 1/2\}$ then: 
\begin{enumerate}
    \item (Theorem 3.11 in \cite{abbe2023generalization}) With probability 1, the training loss converges to 0, and the
coefficient of the learned function $f_{\mathrm{NN}}$ on the
monomial $x_k$ is less than $\eps$, i.e., $\wh{f}_{\mathrm{NN}}(\{ k\}) \leq \eps.$
    \item With probability $1-\exp(-\Omega(n))$ over the random initialization, there exists a time $$t^\star = \Omega\left(\frac{1}{c^L} \log\left(\frac{n}{\gamma-\eps}\right)\right),$$ where $c$ is an absolute constant such that for $0 < t < t^\star$, \Cref{eq:small influence} holds for the function $f_{\mathrm{NN}}(~\cdot~; \theta_t) - f^\star(\cdot)$ with $P = \calU_S$ and $Q = \calU$ and so 
    \[
    \E_{x \sim  Q} (f_{\mathrm{NN}}(x; \theta_t) - f^\star(x))^2
    \leq C \cdot 
    \E_{x \sim  P} (f_{\mathrm{NN}}(x; \theta_t) - f^\star(x))^2\,,
    \]
    for some absolute constant $C,$ and, for $t > t^\star$, the condition of \Cref{eq:small influence} fails to hold.
\end{enumerate}
\end{proposition}

Based on \Cref{thm:DLN-transfer} we conclude that transfer happens as long as $t<t^*$ for some $t^* \gtrsim \log n$ (keeping the rest parameters constant). The proof appears in \Cref{app:proof4}.

%Based on the analysis of \cite{abbe2023generalization}, we have that $\alpha_{\max}$ satisfies $\alpha^2_{\max} \leq \frac{\eps}{2 e^{2 R T_\eps}}$, where $R = |\wh{f}(\{0\}) + \wh{f}(\{k\})| + 1$ and $T_\eps = O(\log(R/\eps)).$
%Under the conditions of the above \Cref{thm:DLN-transfer}, \cite{abbe2023generalization} proved the first Item of the result, i.e., they showed that during the whole training phase, $\wh{f}_{\mathrm{NN}}(\{k\}) \leq \eps$ and the training loss converges to 0 exponentially fast in $t$ in an initialization-dependent way (see Theorem 3.11 in \cite{abbe2023generalization}). 

%Using these results,
%we can establish a connection between (i) the ability for generalization on the unseen of
%diagonal linear networks in the task of learning linear functions in the ``canonical holdout'' setting
%and (ii) the maximum influence of the function $x \mapsto f_{\mathrm{NN}}(x;\theta_t) - f^\star(x)$.
%We note that we restrict our focus on $\eps \in (0,\gamma/2)$ in order to easily assure that there is a gap between $\wh{f}^\star(\{k\})$ and $\wh{f}_{\mathrm{NN}}(\{k\}) \leq \eps$.

\section{Empirical Observations}
\label{sec:exp}
In this section, we discuss further the empirical observations for the continuous variant of our inequalities.

\subsection{Setting} Returning to \Cref{fig:intro} appearing in the introduction, we studied transfer learning with source  distribution $P = \calU([0,1]\times[-1,1])$ and target distribution $Q = \calU([-5,5]^2)$ with true function $f^\star(x,y) = \sin(2\pi x)\sin(2\pi y)$ using as a regressor $f$ (i) a degree-20 polynomial and (ii) a 6-layer size-110 ReLU network. The architecture of the NN consists of 6 layers: the first 5 layers contain 20 nodes each and the last layer 10. 
For model (i), we run polynomial regression and for model (ii), we run AdaGrad. 
In a similar spirit, we perform the same experiment for $f^\star(x,y) = \sin(2\pi x)\sin(2\pi y) + xy$, $P = \calU([-1/2,1/2]^2)$ and $Q = \calU([-5,5]^2)$; the associated results are in \Cref{fig:intro2}. The bottom right plot corresponds to a 6-layer size-110 polynomial NN, i.e., a standard network architecture where we set the activation functions to be some polynomials. For our polynomial network, we use as activation function the polynomial mapping $\sigma(x) = 3 x^2 - 2 x^3$.

\begin{figure}[!ht]
    \centering
\includegraphics[scale=0.28]{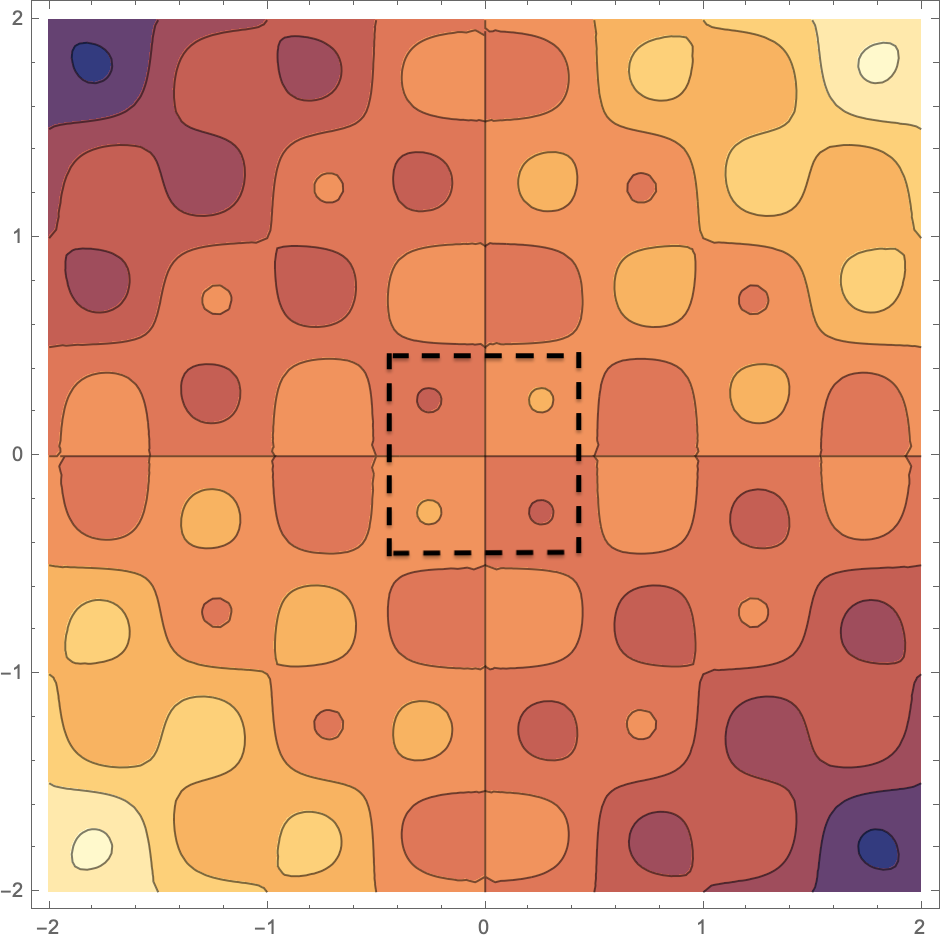}
\includegraphics[scale=0.375]{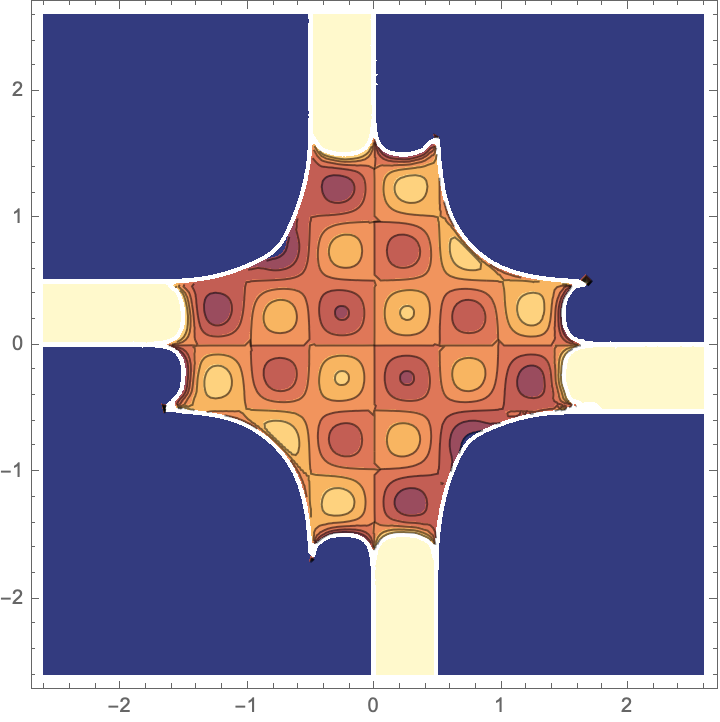}\\
\includegraphics[scale=0.375]{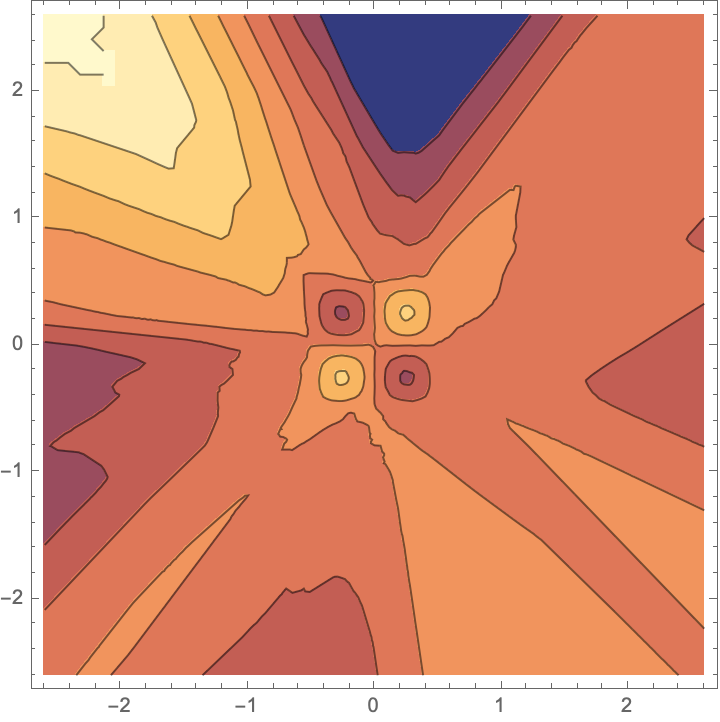}
\includegraphics[scale=0.29]{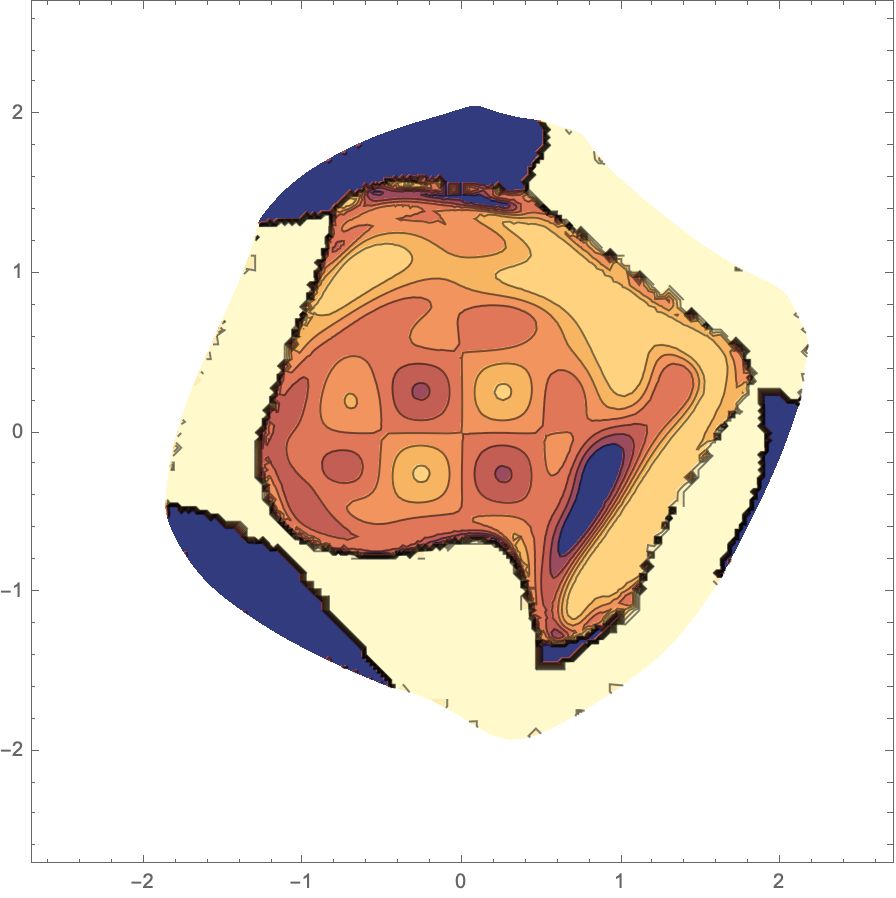}
\caption{ Transfer Learning with source  distribution $P = \calU([-1/2,1/2]^2)$ (see dotted box of top left plot) and the target $Q = \calU([-5,5]^2)$ with true function $f^\star(x,y) = \sin(2\pi x)\sin(2\pi y) + xy$ (top left) using as a regressor $f$ (i) a degree-20 polynomial (top right), (ii) a 6-layer size-110 ReLU neural network (bottom left) and, (iii) a 6-layer size-110 polynomial neural network (bottom right).}
    \label{fig:intro2}
\end{figure}

\subsection{Analysis of Results} Let us now comment on \Cref{fig:intro} and \Cref{fig:intro2}.
Observe that polynomial regression learns everything inside the dotted box, which is completely observable and \emph{non-trivially extrapolates} around that box. On the contrary, the ReLU network perfectly learns the observed box but \emph{fails to extrapolate} around it.

Let us now explain the connection between our Inequality \eqref{eq:main1} with the extrapolation properties of polynomial regression in \Cref{fig:intro} and \Cref{fig:intro2}. First note, that in both examples, $P$ and $Q$ are log-concave (in fact, simply uniform in some rectangle).
Hence, our main inequality \eqref{eq:main1} implies that low-degree polynomials perform transfer learning, i.e., for any $f,p$ $d$-degree polynomials, it holds
$$\E_Q (f - p)^2 \leq C_d \cdot  \|dP/dQ\|_\infty^{2d} \cdot \E_P (f - p)^2.$$ The true function $f^\star$ generating \Cref{fig:intro} (and \Cref{fig:intro2}) is not a polynomial so it is perhaps not immediately clear that doing polynomial regression on samples from $f^\star$ should be enjoying nice transfer learning properties. Yet, thanks to $f^\star$ being smooth on a bounded domain transfer is still happening. To see this, notice that $f^\star$ can be well approximated by a polynomial $p^\star$. Indeed, since we are interested in $L_2$ approximation on a bounded domain, we can use Taylor approximation to have that for an appropriately high $d',$ there is a $d'$-degree polynomial $p^\star$ for which both $\E_P (f^\star - p^\star )^2$ and $\E_Q (f^\star - p^\star )^2$ are arbitrarily small. Let $\hat{f}$ be the output of gradient descent on the polynomial regression task. If  $\hat{f}$ is such that $\E_P (f^\star - \hat{f})^2$ is arbitrarily small, then, by triangle inequality, $\E_P (p^\star - \hat{f})^2$ becomes arbitrarily small as well. Crucially now $p^\star-\hat{f}$ is a $\max\{d',d\}$-degree polynomial. Therefore (as long as no $d',d$ became ``too large'')  our main inequality \eqref{eq:main1} implies $\E_Q (f^\star - p^\star )^2 $ is becoming small.  A final use of triangle inequality concludes that $\E_Q (f^\star - \hat{f} )^2 $ becomes small too. 

Now the trained ReLU network greatly fails to enjoy such nice extrapolation properties in our example. We note that this observed inability of ReLU networks to extrapolate, and in particular, the absence of a similar transfer inequality like \eqref{eq:main1}, is expected based on the literature. In fact, similar failures of ReLU networks (see e.g., \cite{kang2023deep}) have been a topic of discussion of extensive prior work. A large body of work has observed poor generalization and overconfidence of NN when presented with OOD inputs \cite{gulrajani2020search,kang2023deep}. In particular, NN trained with SGD could present extreme cases of \emph{simplicity bias} \cite{shah2020pitfalls} (standard training procedures such as SGD tend to find simple
    models). We leave the extent to which some neural networks (potentially trained with early stopping) could still enjoy some desired transfer properties as an interesting future direction.

\section{Proofs for the results on the Euclidean Domain}
\label{app:proof-euc}
In our proofs, we will use anti-concentration properties of polynomial functions under log-concave distributions, proved by Carbery and Wright \cite{carbery2001distributional}.

\begin{proposition}
[Theorem 8 in \cite{carbery2001distributional}]
\label{prop:cw}
There exists an absolute constant $C$ such that the following holds:
let 
$f : \reals^n \to \reals$ be a non-zero  polynomial of degree at most $d$, $\mu$ be a log-concave probability distribution over $\reals^n$ and
$q, \gamma \in (0,\infty)$. Then it holds that
\[
\Pr_{x \sim \mu} 
\Big [
|f(x)| \leq \gamma 
\Big]
\leq 
\frac{C q \gamma^{1/d} }{ \left( \E_{\mu} |f|^{q/d} \right)^{1/q}}\,.
\]
In particular, for $q = d$, it holds that~
$
\Pr_{x \sim \mu} 
\Big [
|f(x)| \leq \gamma 
\Big]
\leq 
\frac{C d \gamma^{1/d} }{ \left( \E_{\mu} |f| \right)^{1/d}}\,.$
\end{proposition}

\subsection{The Proof of \Cref{theorem:main}}

\label{section:proof of main}

We continue with the proof of \Cref{theorem:main}.

\begin{proof}
For $\alpha \geq 1$, we define the divergence $D_{\alpha}(P \parallel Q) \triangleq \left (\E_{x \sim Q} \left[ \left( \frac{P(x)}{Q(x)} \right)^{\alpha} \right] \right )^{1/\alpha}$; we denote both distributions and associated density functions by $P,Q$ overloading the notation.

Let us fix some $\mu \in \calL$ whose support contains both that of $Q$ and that of $P$. We will overload the notation and denote by $P$ (resp. $Q, \mu$) the density of the associated distribution.
The first step of the proof is to upper bound the expectation of $f$ under $Q$ by that under $\mu$. This is done by a simple change of measure and an application of Hölder's inequality with parameters $\alpha,\beta$:
\begin{equation}
\label{eq:1-1}
\E_{Q}|f| = \E_{x \sim \mu} \left[ \frac{Q(x)}{\mu(x)} |f(x)| \right]
\leq 
D_\alpha(Q \parallel \mu) \left(\E_\mu |f|^{\beta}\right)^{1/\beta}
\,.
\end{equation}

As a next step we need to relate expectations under the log-concave measure $\mu$ and under the general measure $P$, which is the non-trivial part of the argument.
We will first lower bound the value of $\E_P |f|^{\beta}$. For some $\gamma > 0$ to be decided, we have that
\[
\E_P |f|^{\beta} \geq \E_{x \sim P} \Big[|f|^{\beta}~ \Big|  |f(x)|^{\beta} \geq \gamma \Big] \Pr_{x \sim P} \Big[|f(x)|^{\beta} \geq \gamma \Big] \geq \gamma \Pr_{x \sim P} \Big[|f(x)|^{\beta} \geq \gamma \Big]\,.
\]
Hence we get that
\begin{equation}
\label{eq:4}
\E_P |f|^{\beta} 
\geq
\gamma \left( 
1 - 
\Pr_{x \sim P} 
\Big[|f(x)|^{\beta} \leq \gamma \Big]
\right)\,.    
\end{equation}
In order to further lower bound the right-hand side of the above inequality, it suffices to obtain an upper bound on the probability mass of the event $\{x \in \reals^n : |f(x)|^{\beta} \leq \gamma \}$ under $P$.
We have that
\[
\Pr_{x \sim P}\Big [|f(x)| \leq \gamma^{1/\beta} \Big]
= \E_{x \sim \mu} \left[ \frac{P(x)}{\mu(x)} ~ 1\{|f(x)| \leq \gamma^{1/\beta}\} \right]\,.
\]
We can upper bound this quantity using Hölder's inequality for $\alpha,\beta \in [1,\infty]$ with $1/\alpha + 1/\beta = 1$:
\begin{equation}
\label{eq:2a}    
\Pr_{x \sim P}\Big [|f(x)|^{\beta} \leq \gamma \Big]
\leq 
D_\alpha(P \parallel \mu)
\left ( 
\Pr_{x \sim \mu} \Big[ 
|f(x)| \leq \gamma^{1/\beta}
\Big ] \right )^{1/\beta}
\end{equation}

Using \Cref{prop:cw} in \eqref{eq:2a}, by picking $q = \beta d$, we get that
\begin{equation}
\label{eq:3}
\Pr_{x \sim P}\Big [|f(x)| \leq \gamma^{1/ 
\beta} \Big]
\leq 
D_{\alpha}(P \parallel \mu)
\left ( 
\frac{C \beta d \gamma^{1/(\beta d)}}{ \left(\E_\mu |f|^\beta \right)^{1/{\beta d}}}
\right )^{1/\beta}
\end{equation}
We are now ready to pick $\gamma$. We choose $\gamma$ so that
$
D_{\alpha}(P \parallel \mu)
\left ( 
\frac{C \beta d \gamma^{1/(\beta d)}}{ \left(\E_\mu |f|^{\beta} \right)^{1/(\beta d)}}
\right )^{1/\beta}
= 1/2\,,
$
which implies that
$
\gamma = \frac{\E_{\mu} |f|^{\beta}}{\left( C \beta d 2^\beta D_\alpha(P \parallel \mu)^\beta \right)^{\beta d}}\,.
$
Using this choice we can return to \eqref{eq:4} and get that
\[
\E_P |f|^\beta \geq \frac{\E_{\mu} |f|^\beta}{2 \left( C d 2^\beta D_\alpha(P \parallel \mu)^\beta \right)^{\beta d}}
\]
We can now complete the proof using \eqref{eq:1-1}, which gives that
\[
\E_Q |f| \leq 
D_\alpha(Q \parallel \mu)
\left( \E_\mu |f|^\beta
\right)^{1/\beta}
\leq 2 C^d d^d 2^{d \beta} D_\alpha(Q \parallel \mu) D_\alpha(P \parallel \mu)^{\beta d} \left(\E_P |f|^{\beta}\right)^{1/\beta}\,.
\]
By minimizing over $\mu \in \calL$, the proof is complete.
\end{proof}

\subsection{The Proof of \Cref{cor:inequality}}
\label{app:proof1}
\begin{proof}
[Proof of \Cref{cor:inequality}]
Following the same steps as in the proof of \Cref{theorem:main} we can get that
\[
\E_P |f|
\geq 
\gamma \left( 1 - \Pr_P \Big [|f(x)| \leq \gamma \Big] \right)\,,
\]
and for $\alpha, \beta $ such that $1/\alpha + 1/\beta = 1$, we have
\[
\Pr_P \Big [|f(x) | \leq \gamma \Big]
\leq 
D_\alpha (P \parallel Q) \cdot \Pr_Q \Big[|f(x)| \leq \gamma \Big]^{1/\beta}\,.
\]
In the above, since for any $x$ with $P(x) > 0$, it holds that $Q(x) > 0$, the densities' ratio does not explode.
Since $Q$ is log-concave, we can apply \Cref{prop:cw} and get 
\[
\Pr_P \Big[|f(x) | \leq \gamma \Big]
\leq 
D_\alpha(P \parallel Q)
\cdot \left( \frac{C d \gamma^{1/d}}{\left(\E_Q |f|\right)^{1/d}} \right)^{1/\beta}\,.
\]
By choosing $\gamma = \frac{\E_Q |f|}{(C d 2^\beta D_\alpha(P \parallel Q)^\beta)^d}$, we can conclude that
\[
\E_Q |f|
\leq 
C^d d^d 2^{\beta d}
\cdot 
D_\alpha(P \parallel Q)^{\beta d}
\cdot \E_P |f|\,.
\]
\end{proof}

\section{Proofs for the Result in the Boolean Domain}
\label{app:proof2}

We now focus on the proof of \Cref{thm:transfer}, that is of our main transfer inequality in the Boolean domain.
\begin{proof}[Proof of \Cref{thm:transfer}]
Let $f(x) = \sum_{S \subseteq [d]} c_S \chi_S$ be a degree-$d$ multilinear function from $\{-1,1\}^n$ to $\reals$ and fix $Q$ and $P$ be two distributions over $\{-1,1\}^n$. We assume that $f, Q, P$ are such that:
(i) $Q$ is product, i.e., $Q = \otimes Q_i$ and satisfies $\E_{x \sim Q_i}[x] = 0, \E_{x \sim Q_i}[x^2] = 1$ and $\E_{x \sim Q_i} [|x|^3] \leq \beta$ for any $i \in [n]$, (ii) $\sum_{|S| > 0} c_S^2 = 1$ and $\sum_{S \ni i} c_S^2 \leq \tau$ for all $i$, and,
 (iii) $P$ corresponds to the distribution $Q$ conditioned on the seen part $S \subseteq \{-1,1\}^n$ of the Boolean hypercube, i.e., $P = Q_S$ for some set $S \subseteq \{-1,1\}^n$. 
We have that
\[
\E_P f^2
\geq
\gamma
\cdot
\left (
1 - 
\Pr_P[f^2 \leq \gamma]
\right)
=
\gamma \cdot 
\left (
1 - 
\Pr_P
\left[ \left(\sum_S c_S \chi_S\right )^2 \leq \gamma \right]
\right)
\,.
\]
Hence, it suffices to upper bound the above right-hand side probability term. Since $Q$ is supported on the Boolean hypercube, the ratio $\|dP/dQ\|_\infty < \infty$ and we can proceed as follows:
\[
\Pr_P\left[ \left(\sum_S c_S \chi_S\right)^2 \leq \gamma \right]
\leq 
\left
\| \frac{d P}{d Q}  \right \|_\infty
\cdot 
\Pr_Q
\left[ \left| \sum_S c_S \chi_S \right| \leq \sqrt{\gamma}\right]\,.
\]
Next we remark that over $\{-1,1\}^n$, we have that
$g = |f| = \sum_S |c_S| \chi_S$ is a multilinear function of degree $d$ and so, by the invariance principle (cf. \Cref{prop:invariance}) applied on $Q$ and $g$, we get that
\[
\Pr_{x \sim Q}
\left[ g(x) \leq \sqrt{\gamma}\right]
\leq 
\Pr_{y \sim \calN(0,I)}
\left[ g(y) \leq \sqrt{\gamma} \right]
+ 
O(d \beta^{1/3} \tau^{1/8d})\,.
\]
This means that
\[
\Pr_{x \sim Q}
\left[ f(x)^2 \leq \gamma \right]
\leq 
\Pr_{y \sim \calN(0,I)}
\left[ f(y)^2 \leq \gamma \right]
+ 
O(d \beta^{1/3} \tau^{1/8d})\,.
\]
Next, we can apply the Carbery-Wright inequality (cf. \Cref{prop:cw}) and get that
\[
\Pr_{x \sim Q}
\left[ f(x)^2 \leq {\gamma}\right]
\leq 
\frac{d \gamma^{1/2d}}{(\E_{y \sim \calN(0,I)} f(y)^2)^{1/2d}}
+ 
O(d \beta^{1/3} \tau^{1/8d})\,.
\]
Observe that
\[
\E_{y \sim \calN(0,I)}[f(y)^2]
=
\E_{y \sim \calN(0,I)}
\left[\sum_{|S| \leq d, |T| \leq d} c_S c_T \chi_S(y) \chi_T(y) \right]
=
c_\emptyset^2
+
\sum_{|S| \leq d} c_S^2
=
\E_{x \sim Q}[f(x)^2]\,,
\]
and so
\[
\Pr_{x \sim Q}
\left[ f(x)^2 \leq {\gamma}\right]
\leq 
\frac{d \gamma^{1/2d}}{(\E_{x \sim Q} f(x)^2)^{1/2d}}
+ 
O(d \beta^{1/3} \tau^{1/8d})\,.
\]
We now have to optimize over $\gamma:$ we pick
\[
\gamma = \E_{x \sim Q}[f(x)^2]/d^{O(d)} \cdot \left(\frac{1}{ \|d P/ d Q\|_\infty } - O(d \beta^{1/3} \tau^{1/8d})\right)^{2d}.
\]
Note that 
(i) since $P$ is the seen part of $Q$, we have that $\| d P/ d Q \|_\infty = 1/Q(S)$, and
(ii) we have, by assumption, that $Q(S)$ is at least the invariance gap.
These two properties imply that
\[
\E_{Q}[f^2]
\leq 
\frac{d^{O(d)}}{(Q(S) - c_1 d \beta^{1/3} \tau^{1/8d})^{2d}} 
\cdot \E_P[f^2]\,,
\]
for some constant $c_1 > 0$,
which completes the proof.
\end{proof}

\section{Conclusion}
In this work, we show general transfer inequalities for low-degree polynomials in the Euclidean domain. We apply this inequality in order to obtain new results in truncated statistics and in-context transfer learning for linear attention layers. We also provide an extension of this inequality to the Boolean domain and study its connections with the GOTU setup of \cite{abbe2023generalization}.

So far, the theory of statistics and statistical learning theory have been traditionally focusing on three axes: \emph{sample complexity} (how many examples are required to learn the class), \emph{expressivity} (what type of functions can be expressed by the class) and \emph{training time} (how much computation is required to learn the class). In all three axes, multiple beautiful and structural results have been proven. We believe that \textbf{transferability} (i.e., which distribution shifts do not deteriorate generalization) is also  a crucial criterion for the statistician. Our work demonstrates that general and structural statistical results can also be proven in this axis, paving the way for hopefully more to come. For instance, we present a few intriguing future directions.

\begin{enumerate}

\item Whether transferability holds or not is a property of (a) the class of estimators used, (b) the pair of source and target distributions $P$ and $Q$, \emph{and}, (c) the training algorithm. While our work focus solely on (1) an (2), understanding the ability
of transfer learning as an \emph{algorithmic property} (i.e., focusing on (3) apart from solely (1) and (2)) is an interesting direction for future work.

    \item A natural question is whether similar inequalities can be obtained beyond regression, for \emph{classification settings}. It is rather clear that the answer is, in general, no; {consider the simple example of learning a threshold function $h^\star(x) = \mathbb 1\{x > t^\star\}, x \in \reals$ with source distribution $P$ and target $Q$; if $\mathrm{supp}(Q) = [t^\star - 1, t^\star + 1]$ and $\mathrm{supp}(P) = [t^\star + c_1,t^\star+2c_1]$ for $c_1 > 1$ then
    any algorithm cannot distinguish $h^\star$ from the function $h=1$ using the source labeled examples and hence
    cannot
    do better than random guessing in the target distribution}. Yet, it would be interesting if, for instance, one could use our inequalities in the context of polynomial regression in agnostic PAC learning by leveraging the connection between regression and classification. As an example of this connection, we refer to the work of \cite{klivans2023testable} and specifically their ``transfer lemma'' (Lemma 2.8).

    \item We establish in \Cref{theorem:main} that general transfer is possible in the continuous domain for polynomial functions. Can we generalize our result beyond polynomial functions?  Interestingly, our experiments suggest ReLU neural networks fail to enjoy such transfer properties, but is this a structural truth or does something else in our experiment lead to this failure? From a mathematical standpoint this quest is also non-trivial, as our proof crucially uses the Carbery-Wright inequality which is proven solely for polynomial functions. We believe exploring more this direction is interesting both from a mathematical and applied point of view.

    \item Finally, our \Cref{thm:transfer} deepens the connections between transferability and key notions from Boolean Fourier analysis. We prove that as long as the maximum influence of the error of an estimator is sufficiently small, the estimator enjoys general transfer properties. The result naturally arises from the invariance principle from Boolean Fourier analysis. From a mathematical standpoint, the maximum influence condition is also necessary for transfer inequalities to hold (e.g., the Dictator function provides a counterexample). Yet, is the maximum influence the key measure that explains in practical scenaria when a given estimator will enjoy desired transfer properties or not? Our results do not answer this instance-specific question and we believe this is a very interesting future direction to explore.  
\end{enumerate} 

\newpage

\bibliographystyle{alpha}
{\footnotesize \bibliography{bib}}

\newpage
\appendix
\section{Deferred Proofs}
\label{sec:proofs}
In this section, we present all the missing proofs from the main body of the paper.

\subsection{The Proof of \Cref{cor:truncated}}
\label{app:truncated}
\begin{proof}
[Proof of \Cref{cor:truncated}]
%Let us consider the single sample negative log-likelihood function $\ell_S(W,x,y)$ defined as
% \[
% \l_S(W, x, y) = 
% \frac{1}{2}(y - p_W(x))^2
% +
% \log 
% \left(
% \int_S e^{-\frac{1}{2}(z - p_W(x))^2} dz
% \right)\,,
% \]
% for $y \in S$. Note that this mapping is well-defined only for $y \in S$, for which it is non-negative.

Fix some $\alpha > 0$.
Let us consider the input feature vectors $x^{(1)},...,x^{(N)}$.
% and consider the set of parameters $\calW_\alpha$ such that
% \[
% \calW_\alpha = \left\{ W : \min_{i \in [N]}
% \calN(p_W(x^{(i)}),1 ; S) \triangleq 
% \frac{1}{\sqrt{2\pi}} 
% \int_S e^{-\frac{1}{2}(z - p_W(x^{(i)}))^2}dz
% \geq \alpha \right\} \,.
% \]
% In words, $W \in \calW_\alpha$ if the normal distributions centered at $p_W(x^{(1)}),...,p_W(x^{(N)})$ assign non-trivial mass to $S$. More to that, let us assume that the true parameter $W^\star$ belongs to that set.

Let us fix a model $f$ and some feature $x^{(i)}$. Then we introduce the polynomial $f_i$ over $\reals$ 
\[
f_i(y) = \frac{1}{2}(y - f(x^{(i)}))^2\,.
\]
It is clear that $f_i$ is a non-negative degree-$2$ polynomial of $y$. 
To get the desired result, we will apply $N$ times our \Cref{cor:inequality}, once for each polynomial $f_i(y), i \in [N]$.
For any $i \in [N]$, fix feature vector $x^{(i)}$ and consider the log-concave target distribution $Q^{(i)} = \calN(f^\star(x^{(i)}),1)$ and source truncated distribution $P^{(i)} = \calN_S(f^\star(x^{(i)}),1)$. 
For the $i$-th case, our \Cref{cor:inequality} implies that
\[
\E_{y \sim Q^{(i)}} f_i(y)  \leq C \cdot \|dP^{(i)}/dQ^{(i)}\|_\infty^2 \cdot \E_{y \sim P^{(i)}} f_i(y)\,,
\]
for some absolute constant $C>0.$
Note that we have that $\|dP^{(i)}/dQ^{(i)}\|_\infty = 1/\alpha$ for any $i \in [N]$ under the mass assumption for $S$.
Let us finally take the uniform average of the $N$ inequalities and get
\[
\frac{1}{N} \sum_{i \in [N]}
\E_{y \sim \calN(f^\star(x^{(i)}),1)} f_i(y)
\leq (C/\alpha^2) \cdot 
\frac{1}{N} \sum_{i \in [N]}
\E_{y \sim \calN_S(f^\star(x^{(i)}),1)} f_i(y)\,.
\]
% The left hand side can be written as
% \[
% \frac{1}{N} \sum_{i \in [N]}
% \E_{y \sim \calN(p_{W^\star}(x^{(i)}),1)}
% \frac{1}{2}(y - p_W(x^{(i)}))^2
% + 
% \log(\sqrt{2\pi})
% +
% \left( 
% c_\alpha
%  +
% \frac{1}{N}
% \sum_{i=1}^N \log (\calN(p_W(x^{(i)}),1;S))
% \right)\,, 
% \]
% which means that
% \[
% \ell(W) \leq 
% (C/\alpha^2) \cdot \ell_S(W)
% + 
% C(1/\alpha^2 - 1) c_\alpha
% +
% \log(1/\alpha)\,.
% \]
\end{proof}

\subsection{The Proof of \Cref{thm:icl transfer}}
\label{app:proof3}

\begin{proof}
[Proof of \Cref{thm:icl transfer}]
Consider two distributions $P_H, Q_H$ over $\reals^n$ (distributions over the normal vector of the linear function) and 
$P_X, Q_X$ over $\reals^n$ (distributions over the feature space).
Let us fix the source distribution $P = P_H \otimes P_X^{\otimes (N+1)}$ and the target distribution $Q = Q_H \otimes Q_X^{\otimes (N+1)}$\footnote{For simplicity we assume that $Q_X = Q_X^\mathrm{query}$ but, due to independence of the random variables, the proof could go through even if they are different.}.
Note that each sample from $P$ or $Q$ can be seen as an in-context prompt.

Let us set $x = (x_1,...,x_N,x_{\mathrm{query}})$ and consider the squared loss between the linear function $w^\top x_{\mathrm{query}}$ and the prediction of a linear single-layer self-attention model, i.e., $f(x,w) = (u^\top H u - w^\top x_{\mathrm{query}})^2$; recall \eqref{eq:polynomial-lsa} and note that $H$ is a matrix where each entry is a low-degree polynomial of $x$ and $w$ and $u$ is only depending on the trainable parameters of the LSA model, which are considered fixed in the upcoming analysis.
We hence remark that the degree of $f$ is $d = O(1)$. Our goal is to relate the source loss $\E_{P}[f(x,w)]$ with the target loss $\E_{Q}[f(x,w)]$.

Item (1) in \Cref{thm:icl transfer} follows directly 
from \Cref{theorem:main} applied on $P$ and $Q$ with $\beta = 1$ and $\alpha = \infty$.

For Items (2) and (3),
we will apply the proof technique of \Cref{theorem:main} on an appropriate product subspace $\reals^{n} \otimes ... \otimes \reals^n$ with $\beta =1$ and $\alpha = \infty$. We will focus on task shift (Item (2)); similar steps would yield the result for Item (3), thanks to the independence of the random variables. For the case of task shift, we proceed as follows.

Let us fix some $\mu \in \calL_n$, which is a log-concave measure over $\reals^n$. We will overload the notation and denote by $P_H$ (resp. $Q_H, \mu$) the density of the associated distribution.
The first step of the proof is to upper bound the expectation of $f$ under $Q = Q_H \otimes Q_X^{\otimes N+1}$. Note that we will handle only a distribution shift for the task distribution and so, thanks to the product structure of $Q$:
\begin{align}
\E_{(x,w) \sim Q}[f(x,w)] &= \E_{x \sim Q_X^{\otimes N+1}} \E_{w \sim \mu} \left[ \frac{Q_H(w)}{\mu(w)} f(x,w) \right]\\
& \leq 
\| d Q_H/ d \mu \|_\infty \left(\E_{x \sim P_X^{\otimes N+1}} \E_\mu [f(x,w)]\right)
\label{eq:1a}
\,,
\end{align}
where in the last inequality we also used that $Q_X = P_X$.

As a next step we need to relate expectations under the log-concave measure $\mu$ and under the general measure $P_H$. For some $\gamma > 0$ to be decided, we have that
\begin{align*}
\E_{(x,w) \sim P} [f(x,w)] 
& \geq 
\E_{(x, w) \sim P} 
\Big[f(x,w)~ \Big|  f(x,w) \geq \gamma \Big] \Pr_{(x,w) \sim P} \Big[ f(x,w) \geq \gamma \Big] 
\\ 
& \geq \gamma 
\Pr_{(x,w) \sim P} \Big[f(x,w) \geq \gamma \Big]\,.
\end{align*}
Hence we get that
\begin{equation}
\label{eq:4a}
\E_{(x,w) \sim P} [f(x,w)] 
\geq
\gamma \left( 
1 - 
\Pr_{(x,w) \sim P} 
\Big[f(x,w) \leq \gamma \Big]
\right)\,.    
\end{equation}
It suffices to obtain an upper bound on the probability mass of the event $\{(x,w) \in (\reals^n)^{\otimes N+2} : f(x,w) \leq \gamma \}$ under $P$.
Under distribution shift from $P_H$ to $\mu$, we have that
\[
\Pr_{(x,w) \sim P}\Big [f(x,w) \leq \gamma \Big]
= \E_{x \sim P_X^{\otimes N+1}}\E_{w \sim \mu} \left[ \frac{P_H(w)}{\mu(w)} ~ 1\{f(x,w) \leq \gamma\} \right]\,.
\]
We can upper bound this quantity as follows:
\begin{equation}
\label{eq:2aa}    
\Pr_{(x,w) \sim P}\Big [f(x,w) \leq \gamma \Big]
\leq 
\|d P_H/ d \mu\|_\infty
\left ( 
\Pr_{x \sim P_X^{\otimes N+1}, w \sim \mu} \Big[ 
f(x,w) \leq \gamma
\Big ] \right )\,.
\end{equation}
Since $x$ and $w$ are independent and the product of two log-concave densities is log-concave, we can apply \Cref{prop:cw} in \eqref{eq:2aa}, by picking $q = d$ (the degree of $f)$, and get that
\begin{equation}
\label{eq:3a}
\Pr_{(x,w) \sim P}\Big [f(x,w) \leq \gamma \Big]
\leq 
\|d P_H / d\mu\|_\infty
\left ( 
\frac{C d \gamma^{1/ d}}{ 
\left(
\E_{
x \sim P_X^{\otimes N+1}, w \sim \mu}
[f(x,w)] 
\right)^{1/{d}}}
\right )\,.
\end{equation}
We choose $\gamma$ so that
\[
\|d P_H / d \mu \|_\infty
\left ( 
\frac{C d \gamma^{1/ d}}{ \left(\E_{x \sim P_X^{\otimes N+1}, w \sim \mu} [f(x,w)] \right)^{1/d}}
\right )
= 1/2\,,
\]
which implies, since $d$ is some absolute constant, that there exist some universal positive constants $C,c$ so that
\[
\gamma = \frac{\E_{x \sim P_X^{\otimes N+1}, w \sim \mu} [f(x,w)]}{ C  \| dP_H / d \mu \|_\infty^c}\,.
\]
Using this choice we can return to \eqref{eq:4a} and get that, for some absolute constants $C,c$,
\[
\E_{x \sim P_X^{\otimes N+1}, w \sim P_H} [f(x,w)] \geq 
\frac{\E_{x \sim P_X^{\otimes N+1}, w \sim \mu} [f(x,w)]}{C  \| dP_H / d \mu \|_\infty^c}\,.
\]
We can now complete the proof using \eqref{eq:1a}, which gives that
\[
\E_{(x,w) \sim Q} [f(x,w)] \leq 
C \cdot 
\left \| \frac{d Q_H}{d \mu} \right\|_\infty
\left \| \frac{d P_H}{d \mu} \right\|_\infty^{c}
\cdot 
\E_{(x,w) \sim P}[f(x,w)]
\,.
\]
By minimizing over $\mu \in \calL_n$, the proof for the task shift (Item (2)) is complete. The proof for Item (3) is similar.
\end{proof}

\subsection{The Proof of \Cref{thm:DLN-transfer}}
\label{app:proof4}

\begin{proof}[Proof of \Cref{thm:DLN-transfer}]
Let us assume that $|\wh f^\star(\{k\})| > \gamma > 0$ and pick any $\eps \in (0,\gamma/2)$. Given that $\eps$, we can tune the initialization of the weights of the model as described in the Theorem. 
In particular, we independently draw the random variable
\[
w_i^{(\ell)}(0) \sim \mathcal{U}(-\alpha, \alpha), i \in [n], \ell \in [L]\,,
\]
where $\alpha = \min\{1/2, \alpha_{\max}\}$ and $\alpha_{\max} = \left((L-2) R T_\eps + (8/\eps)^{(L-2)/L}\right)^{1/(2-L)}, R = 1 + |\wh{f}^\star(\{0\}) + \wh{f}^\star(\{k\})|, T_\eps = c \log(R/\eps)$, as defined by \cite{abbe2023generalization}. 

For $t > 0$,
recall that
 \[
    \theta_t = (b_{\mathrm{NN}}(t), w_1^{(1)}(t),...,w_{n}^{(1)}(t),...,w_{1}^{(L)}(t),...,w_{n}^{(L)}(t))\,,
    \]
    and
    \[
    f_{\mathrm{NN}}(x;\theta_t)
    =
    b_{
    \mathrm{NN}
    }(t)
    + 
    \sum_{i \in [n]}
    \left( \prod_{\l = 1}^L w_i^{(\l)}(t) \right) x_i\,.
    \]
Let us consider the function $f_{\mathrm{NN}}(\cdot; \theta_t) - f^\star(\cdot)$, which is equal to
\[
f_{\mathrm{NN}}(x; \theta_t) - f^\star(x) = \sum_{S \subseteq [n]} (\wh{f}_{\mathrm{NN}}(S; \theta_t) - \wh{f}^\star(S) ) \chi_S = 
(b_{\mathrm{NN}}(t) - b^\star)
+ \sum_{i=1}^n
\left( \prod_{\l = 1}^L w_i^{(\l)}(t) - \wh{f}^\star(\{i\}) \right) x_i\,.
\]
We consider the setting of \cite{abbe2023generalization}[Theorem 3.11] where 
$P = \calU_S$ and $Q = \calU$.
Since we have that $d = 1, \beta = O(1)$ and $Q(S) = 1/2$, the condition requires that the maximum influence 
of $f_{\mathrm{NN}}(\cdot; \theta_t) - f^\star(\cdot)$, denoted as
$\tau(t)$, is at most $c_0 < 1$, for some absolute constant $c_0$, where 
\begin{align}
\label{eq:ratio}
\tau(t) &= \max_{i \in [n]} \mathrm{Inf}_i(f_{\mathrm{NN}}(\cdot; \theta_t) - f^\star(\cdot))
=
\max_{i \in [n]}
\E_{x \sim \calU(\{-1,1\}^n)}
\Var_{x_i \sim \mathrm{Rad}(1/2)}
[f_{\mathrm{NN}}(x;\theta_t) - f^\star(x)]
\\
& 
= \frac{\max_{i \in [n]} \left( \prod_{\l = 1}^L w_i^{(\l)}(t) - \wh{f}^\star(\{i\}) \right)^2 }
{\sum_{i \in [n]} \left( \prod_{\l = 1}^L w_i^{(\l)}(t) - \wh{f}^\star(\{i\}) \right)^2}
\triangleq \frac{\max_{i \in [n]} \Delta(i, t)}{\sum_{i \in [n]} \Delta(i,t)}
\,,
\end{align}
where we normalize so that the conditions of the invariance principle are satisfied.
At time $t=0$, using the random initilization described above,
we have that
\[
\tau(0)
= 
\frac{\Delta(i^\star(0), 0)}{\sum_{i \in [n]} \Delta(i, 0)}\,,
\]
where $i^\star(0)$ is the index with the largest gap at the initialization. 
Our goal is to show that $\tau(0) < c_0$ with high probability over the initialization.

Since the Fourier coefficients of $f^\star$ are $O(1)$ as $n$ increases, it suffices to show  that the sum of the random variables $\Delta(i,0), i \in [n]$ will be at least $c n$ with probability $1 - \exp(-\Omega(n))$ for some absolute constant $c$. 
To be more specific, consider the random variable
\[X_i = 1\{\Delta(i,0) \geq \beta_0^2\}\] 
for $i \in [n]$, where $\beta_0$ to be decided. 
The event $\{X_i=1\}$ corresponds to the ``good'' case where the initialization of the weights in direction $i$ and the true Fourier coefficient differ by an absolute constant.

Note that $X_1,...,X_n$ are independent Bernoulli random variables with \[ \E[X_i] = 1 - 
\Pr[\Delta(i,0) < \beta_0^2] = 
1 - \Pr_{w_i^{(1)}(0),...,w_i^{(L)}(0)}
\left[  \prod_{\ell=1}^L w_i^{(\ell)}(0)  \in [\wh{f}^\star(\{i\}) - \beta_0, \wh{f}^\star(\{i\}) + \beta_0] \right]
\,,\]
The random variable $Y_i = \prod_{\ell=1}^L w_i^{(\ell)}(0)$ is a product of $L$ independent uniform random variables over $(-\alpha,\alpha)$ and so $Y_i \in (-\alpha^L, \alpha^L), L \geq 2$. The cumulative density function of $Y_i$ at $z$ for $L=2$ is
\begin{align*}
F_{Y_i}(z) & =
\Pr[w_i^{(1)} w_i^{(2)} \leq z]
=
\int_{-\alpha}^{\alpha}
\Pr[w_i^{(2)} \leq z/x] f_{w_i^{(1)}}(x) dx = \\
& =
\int_{-\alpha}^z \frac{dx}{2\alpha}
+
\int_{z}^\alpha 
(-\alpha + z/x)
\frac{dx}{(2\alpha)^2}
= (z+\alpha)/(2\alpha) - \alpha(\alpha-z)/(2\alpha)^2
-z\log(z/\alpha)/(2\alpha)^2\,.
\end{align*}
This means that the probability density function for $L=2$ is $f_{Y_i}(z) = C_{1,\alpha} - C_{2,\alpha}\log(z)$ for some constants $C_{1,\alpha}, C_{2,\alpha}$ and $Y_i$ is supported on $(-\alpha^2,\alpha^2).$
A similar expression can be obtained by induction on $L$ for the case $L > 2$.
This means that
\[
\E[X_i] = 1 - \mathbb{1}\{\wh{f}^\star(\{i\}) \in (-\alpha^L, \alpha^L)
\} \cdot p_i
\]
where
\[
p_i(\alpha,L,\wh{f}^\star(\{i\}), \beta_0)= \int_{\wh{f}^\star(\{i\}) - \beta_0}
^{\wh{f}^\star(\{i\}) + \beta_0}
f_{Y_i}(z) dz\,.
\]
Note that since we initialize each weigth with a uniform distribution over $(-\alpha,\alpha)$, the probability $p_i$ decreases as $|\wh{f}^\star(\{i\})|$ increases.
Let us pick $\beta_0 = \alpha^L/1000$.
If $|\wh{f}^\star(\{i\})| > \alpha^L$, then we are done, since $\Delta(i,0)$ is at least some absolute constant for any $i \in [n]$. Otherwise, 
having fixed $\eps, L$, the initialization interval $(-\alpha,\alpha)$ is also fixed (independent of $n)$ and so the probability of the ``bad'' event $p_i$ is a constant depending on $\alpha,L,\wh{f}^\star(\{i\})$ which is bounded away from $1$.

Recall that if $\sum_{i = 1}^n X_i > c n$, for some absolute constant $c$, then $\tau(0) < c_0$ for sufficiently large $n$. We have that
\[
\mu = \E \sum_{i =1}^n X_i = n - \sum_{i =1}^n
{1\{\wh{f}^\star(\{i\}) \in (-\alpha^L,\alpha^L)\}}
p_i(\alpha,L,\wh{f}^\star(\{i\}), \alpha^L/1000)
= C n\,, 
\]
for some absolute constant $C = C_{\alpha,L,f^\star}$.
Hoeffding's inequality implies that
\[
\Pr\left[  \left| \sum_{i =1}^n X_i - \mu \right| > \mu/10 \right] \leq 2 \exp(- \mu^2/ 50)\,.
\]
This means that with probability $1 - \exp(- Cn)$, it holds that 
$\sum_{i =1}^n X_i > cn$ for some absolute constants $C,c$.
Conditioning on this event, we get that
\[
\tau(0) \leq \frac{C}{n} < c_0
\]
for some absolute constant $C > 0$. It remains to show that there exists $t^\star > 0$ such that $\tau(t) < c_0$ for $t < t^\star$ and $\tau(t) > c_0$ for $t > t^\star$.

As gradient flow evolves, it is shown in the proof of Theorem 3.11 in \cite{abbe2023generalization} that
\begin{equation}
    \label{eq:2}
    \wh{f}_{\mathrm{NN}}(\{k\})(t) \leq \eps ~~ \forall t\,,
\end{equation}
and, for any $i \neq k$, it holds that 
\begin{equation}
\Delta(i,t)
\leq 
\Delta(i,0)
O(e^{-2 c_i^{L-1} t})\,,
\end{equation}
where, for fixed $i \in [n]$, $c_i$ is defined as follows: first let $j_i^\star = \argmin_{j \in [L]} |w_i^{(j)}(0)|$ and take $c_i = \min_{j \neq j_i^\star \in [L]} (w_i^{(j)}(0))^2 - (w_i^{(j_i^\star)}(0))^2 ( > 0$ with high probability). Moreover, $\Delta(i,t)$ is a decreasing function of $t$ for $i \neq k$.

From \eqref{eq:2} we know that $\Delta(k,t) = \Omega(1)$ with probability 1 since the two Fourier coefficients have a constant gap.
Since for $i \neq k,$ the function $\Delta(i,t)$ is decreasing with $t$, it will hold that, with probability 1, there exists some $t_k$ such that
$i^\star(t) = k$ for any $t > t_k$, i.e, the numerator of $\tau(t)$ is always equal to $\Delta(k,t)$ for $t > t_k$.

Fix a parameter $\eta$ sufficiently small. We will say that the diagonal linear NN has learned $f^\star$ in direction $i \neq k$ if
for some time $t_i$, it holds that $\Delta(i, t) \leq \eta$ for any $t > t_i$.
Clearly, the diagonal linear NN cannot learn $f^\star$ in direction $k.$

Let us define $t^\star$ as the first time that $\tau(t^\star) > c_0$. We mention that $t^\star$ exists and is finite since $\Delta(i,t)$ are decreasing in $t$ for $i \neq k$ but $\Delta(k,t)$ is not. Hence, $\tau(t) \to 1$ as $t \to \infty$.

This means that for any $t < t^\star$, the condition of \Cref{thm:transfer} holds. On the other side, when $t > t^\star$, our transfer inequality cannot be applied.

Having conditioned on the event that $\sum_{i =  1}^n \Delta(i,0) > cn$, there exists an absolute constant $\beta$ such that if $I_0$ is the set of indices $i \neq k$ such that $\Delta(i,0) > \beta$, then $|I_0| = \Omega(n)$ (since each $\Delta(i,0)$ is $O(1))$.
We are going to specify a lower bound for the time value $t^\star.$ It is necessary for $t^\star$ to be larger than a time $t_{\mathrm{LB}}$ such that that
$\sum_{i \in I_0} \Delta(i,t_{\mathrm{LB}}) = o(n);$ this means that at time $t_{\mathrm{LB}} < t^\star$ there must exist $i \in I_0$
such that
$\Delta(i,t_{\mathrm{LB}}) \leq C (\gamma-\eps)^2/n$, where $(\gamma-\eps)^2$ is a lower bound on $\Delta(k,t_{\mathrm{LB}})$. This means that it must be the case that 
$\Delta(i,0) e^{-2 c_i^L t_{\mathrm{LB}}} \leq C (\gamma-\eps)^2/n$, which implies that $t_{\mathrm{LB}} \geq \frac{C_i}{2c_i^L} \log\left(\frac{n \Delta(i,0)}{C (\gamma-\eps)^2}\right)$ for some absolute constants $C_i, c_i, C$. Then 
it holds that
\[
t^\star \geq \frac{C}{2c^L} \log(\beta n / (\gamma-\eps))\,,
\]
for some absolute constants $C,c > 0$.
\end{proof}

\section{Transfer Coefficient Examples}
\label{sec:examples}
In this section, we give a couple of examples to better understand the transfer coefficient of \Cref{theorem:main}.

We first consider the setting where both $P$ and $Q$ are standard Gaussians with means $0$ and $\mu$ respectively in the single-dimensional setting. We show that the ratio $\|d Q/d P\|_\infty$ scales exponentially with the distance of the two means. In the contrary when we consider the ratio $\inf_{\nu} \|dQ/d\nu\|_\infty \|dP /d\nu\|_\infty$ the scaling drops to polynomial, where the infimum runs over log-concave measures. 

\begin{lemma}
[Transfer Coefficient for Gaussians in 1D]
\label{lemma:gaussian1D}
Consider target distribution $Q = \calN(\mu, 1)$ and source distribution $P = \calN(0, 1)$ where $\mu \neq 0$. 
Then it holds that
$\|dQ / dP\|_\infty = \Omega(\exp(\mu^2))$
but 
there exists a log-concave probability distribution $\nu$ such that
$\|dQ / d\nu\|_\infty \cdot \|d P/ d\nu\|_\infty  = 1 + O(\mu^2).$
\end{lemma}
\begin{proof}
We have that
$\|d Q/ dP\|_{\infty} = \Omega(\exp(\mu^2/2))$ by setting $x = \mu$.
Without loss, assume that $\mu > 0.$
Consider the function 
\[
f(x) = 1\{x \leq 0\} \calN(0,1; x) + 1\{x \in (0,\mu) \} \calN(0,1;0) + 1\{x \geq \mu\} \calN(\mu,1;x)\,,\] 
where $\calN(\mu,1;x) = \frac{1}{\sqrt{2\pi}} e^{-(x-\mu)^2/2}$. 
Clearly, $f$ is not a density and, in order define the desired log-concave distribution $\nu$ so we have to compute the value
\[
Z_\nu = \int_{\reals} f(x) dx
= 1 + \int_{0}^\mu \calN(0,1;0) dx = 1 + \frac{\mu}{\sqrt{2\pi}}\,.
\]
Let us consider the measure $\nu$ with density $f/Z_\nu$.
It remains to argue about the value of each density ratio.
We start with the case $d\calN(0,1)/d \nu.$
\begin{itemize}
    \item If $x \leq 0$, we have that
    \[
    \frac{d \calN(0,1)}{d \nu}(x) = Z_\nu = 1 + \mu/\sqrt{2\pi}\,.
    \]
    \item  If $x \geq \mu$, we have that
    \[
    \frac{d \calN(0,1)}{d \nu}(x) = Z_\nu \cdot \frac{\exp(-x^2/2)}{\exp(-(x-\mu)^2/2)}\,.
    \]
    Since we take the supremum over $x \geq \mu$, we observe that the numerator $\exp(-x^2/2)$ decays much faster than the denominator $\exp(-(x-\mu)^2)$ in the interval $(\mu, \infty)$ and so the density ratio in this regime is upper bounded by $1 + \mu/\sqrt{2\pi}.$

    \item Finally, if $x \in (0,\mu)$, then 
    \[
    \frac{d \calN(0,1)}{d \nu}(x) = Z_\nu \cdot \exp(-x^2/2)\,.
    \]
    This means that taking the supremum over $(0,\mu)$ we get a density ratio upper bounded by $1 + \mu/\sqrt{2\pi}$.
\end{itemize}
The cases for $\calN(\mu,0)$ are symmetric. Hence in total we have that the transfer coefficient
\[\|d \calN(0,1)/d\nu\|_\infty \|d \calN(\mu,1)/d\nu\|_\infty = O((1+\mu)^2)\,.
\]
We remark that the function $\phi(x) = 1\{x \leq 0\} x^2 + 1\{x \geq \mu\} (x-\mu)^2$ is continuous and convex and so $\nu$ is log-concave.
\end{proof}

We can also extend this construction in high dimensions.
\begin{lemma}
[Transfer Coefficient for Gaussians in High Dimensions]
Consider target distribution $Q = \calN(\mu, I)$ and source distribution $P = \calN(0, I)$ where $\mu \in \reals^n$ with $\|\mu\|_2 > 0$. 
Then 
there exists a log-concave probability distribution $\nu$ such that
$\|dQ / d\nu\|_\infty \cdot \|d P/ d\nu\|_\infty  = O((1+\|\mu\|_2)^2).$
\end{lemma}
\begin{proof}
Without loss of generality let us assume that $\mu = \gamma e_1 \in \reals^n, \gamma > 0.$
Let us define the Gaussian density $\calN(\mu,I;x) = (2\pi)^{-n/2} \exp(-\|x-\mu\|_2^2/2)$.
We consider the mapping 
\[
f(x) = 1\{x_1 < 0\} \calN(0,I; x) + 1\{x_1 > \gamma\} \calN(\mu,I; x) + 1\{x_1 \in (0,\gamma)\} \calN(0,I; (0,x_{-1}))\,,
\]
where $(0,x_{-1})$ corresponds to the $n$-dimensional vector $x$ with its first coordinate fixed to 0. In order to make $f$ a density function that corresponds to a distribution $\nu$, we have to normalize with the integral
\[
Z_\nu = 
\int_{\reals^n} f = 1 + \int_{x_1=0}^\gamma \int_{x_{-1}} \calN(0,I; (0, x_{-1}))dx = 1 + \gamma/\sqrt{2\pi}\,.
\] 
Let us take $\nu$ be the log-concave distribution with density $f/Z_\nu$.
Then using a similar case analysis as in the 1-dimensional proof of \Cref{lemma:gaussian1D}, we get that the transfer coefficient is of order 
$\|d\calN(0,I)/d\nu\|_\infty \|d\calN(\mu, I)/d\nu\|_\infty = O((1+\|\mu\|_2)^2)$.
\end{proof}
We remark that the way we design the distribution $\nu$ in the above lemma can be adapted to more general log-concave distributions $P$ and $Q$ where $Q$ is a translated version of $P$ and $P$ is a product measure.
\begin{lemma}
[Transfer Coefficient for Translated Log-concave Products]
Consider a source distribution $P$ with density $P(x) = \prod_{i \in [n]} \phi_i(x_i)$, where each $\phi_i$ is log-concave with mode at 0 and $\phi_i(0) = \Theta(1)$ and let $Q$ be a translation of $P$ in direction $\mu \in \reals^n$. Then there exists a log-concave probability distribution $\nu$ such that
$\|d P/d\nu\|_\infty \|d Q/d\nu\|_\infty = O((1+\|\mu\|_2)^2)$.
\end{lemma}
\begin{proof}
Let $\mu = \gamma e_1, \gamma > 0$. Then define 
$f(x) = 1\{x_1 < 0\} P(x) + 1\{x_1 > \gamma\} P(x-\gamma e_1) + 1\{x_1 \in (0,\gamma)\} P (0,x_{-1})$, where $(0,x_{-1})$ corresponds to the $n$-dimensional vector $x$ with its first coordinate fixed to 0. In order to make $f$ a density function we have to normalize with the integral $\int_{\reals^n} f = 1 + \int_{x_1=0}^\gamma
\phi_1(0)
\int_{x_2} \phi_2(x_2) 
\cdots
\int_{x_n} \phi_n(x_n) 
= 1 + \int_{x_1=0}^\gamma \phi_1(0) = 1 + \Theta(\gamma)$. The result follows by a similar case analysis as in \Cref{lemma:gaussian1D}.
\end{proof}

We mention that the above results can be extended for translations of log-concave densities over $\reals^n$ which are not necessarily products. As a simple ilustration, consider $P = \calN(0,\Sigma)$ and $Q = \calN(\gamma e_1, \Sigma)$ for some arbitrary full-rank $\Sigma$ and $\gamma > 0$. Then, we can consider 
$f(x) = 1\{x_1 < 0\} \calN(0, \Sigma; x) + 1\{x_1 > \gamma\} \calN(\gamma e_1, \Sigma; x) + 1\{x_1 \in (0,\gamma)\} \calN(0, \Sigma; (0,x_{-1}))$, where $(0,x_{-1})$ corresponds to the $n$-dimensional vector $x$ with its first coordinate fixed to 0. 
In order to make $f$ a density function we have to normalize with the integral $\int_{\reals^n} f = 1 + \int_{x_1=0}^\gamma \int_{x_{-1}} \calN(0, \Sigma; (0, x_{-1}))dx$. Note that 
\[\
\calN(0, \Sigma; (0, x_{-1})) = \frac{1}{\sqrt{(2 \pi)^n \mathrm{det}(\Sigma)}} e^{- x_{-1}^\top (\Sigma')^{-1} x_{-1}}\,, 
\]
where $\Sigma' \in \reals^{(n-1) \times (n-1)}$ is obtained from $\Sigma$ by deleting the first row and column. This matrix is still symmetric and PSD and so induces an $(n-1)$-dimensional Gaussian distribution. Hence, we have that $\int_{\reals^n} f = 1 + \gamma \cdot \sqrt{ \frac
{(2\pi)^{n-1} \mathrm{det}(\Sigma')}
{(2\pi)^{n} \mathrm{det}(\Sigma)}
}
= 
1 + \frac{\gamma}{\sqrt{2\pi}}
\cdot \sqrt{\frac{\mathrm{det}(\Sigma')}{\mathrm{det}(\Sigma)}}\,.
$
Let $\lambda_i$ (resp. $\lambda_i'$) be the $i$-th eigenvalue of $\Sigma$ (resp. $\Sigma'$) and $\lambda_1 > 0$.
By the Eigenvalue Interlacing Theorem (since $\Sigma'$ is a principal submatrix of $\Sigma$), we know that 
$\frac{\mathrm{det}(\Sigma')}{\mathrm{det}(\Sigma)} = \frac{1}{\lambda_1} \cdot \prod_{i = 1}^{n-1} \frac{\lambda_i'}{\lambda_{i+1}} \leq 1/\lambda_1$. Hence, for non-degenerate Gaussian distributions, we get that the partition function $\int_{\reals^n} f = 1 + O(\gamma)$. Similar analysis as in the above lemmata gives that the log-concave distribution with density $f/\int_{\reals^n}f$ satisfies $\|d\calN(0,\Sigma)/d\nu\|_\infty \|d\calN(\gamma e_1, \Sigma)/d\nu\|_\infty$ scales as $(1 + c\gamma)^2$ for some constant $c$, which depends on the covariance matrix $\Sigma$. 

A similar results can be established for  translations of general log-concave distributions.

\begin{remark}
[Transfer for Families of Target Distributions]  
It is interesting to study scenarios where the target distribution $Q$ is not fixed but belongs to some class of distributions. For instance,
consider 
a family of target distributions $\calQ = \{ \calN(\mu,I) : \mu \in \calW\}$ for some $\calW \subseteq \reals^n$ and source distribution $P = \calN(\mu', I)$ for some $\mu' \in \reals^n$. Then it holds that 
the cost of transfer learning from $P$ to the family $\calQ$ scales as $\sup_{\mu \in \calW} (1+\|\mu - \mu'\|_2)^2$, since for any distribution in $\calQ$, there exists a (potentially different) log-concave distribution that achieves a transfer coefficient of order at most $\sup_{\mu \in \calW} (1+\|\mu - \mu'\|_2)^2$.
\end{remark}